\documentclass{article}
\usepackage{ijcai24}

\usepackage{times}
\usepackage{soul}
\usepackage{url}
\usepackage[hidelinks]{hyperref}
\usepackage[utf8]{inputenc}
\usepackage[small]{caption}
\usepackage{graphicx}
\usepackage{amsmath}
\usepackage{amsthm}
\usepackage{booktabs}
\usepackage[switch]{lineno}

\usepackage[ruled,vlined]{algorithm2e}
\SetAlgorithmName{Mechanism}{mechanism}{List of mechanisms}

\usepackage{amsfonts}
\usepackage{multirow}
\usepackage{array,multirow}
\usepackage[disable]{todonotes}
\usepackage{thmtools, thm-restate}
\usepackage{enumitem}
\usepackage[multiple]{footmisc}
\usepackage{subcaption}
\usepackage{xcolor}

\usepackage{cleveref} 
\crefname{algorithm}{mech.}{mechs.}
\Crefname{algorithm}{Mechanism}{Mechanisms}

\newcommand{\citet}[1]{\citeauthor{#1} [\citeyear{#1}]}

\allowdisplaybreaks 

\DeclareMathOperator*{\argmax}{argmax}

\newtheorem{proposition}{Proposition}
\newtheorem{corollary}{Corollary}

\theoremstyle{definition}
\newtheorem{definition}{Definition}

\theoremstyle{remark}
\newtheorem*{remark}{Remark}

\newcommand{\start}{s}
\newcommand{\goal}{g}
\newcommand{\agentpath}{\pi}
\newcommand{\cost}{c}
\newcommand{\reportedcost}{\hat{\cost}}
\newcommand{\val}{v}
\newcommand{\reportedval}{\hat{\val}}
\newcommand{\allocrule}{f}
\newcommand{\payment}{p}
\newcommand{\vcgfunc}{h}
\newcommand{\range}{S}
\newcommand{\agenttype}{\tau}
\newcommand{\reportedtype}{\hat{\agenttype}}

\newcommand{\welfare}{w}
\newcommand{\reportedwelfare}{\hat{\welfare}}
\newcommand{\outcome}{d}
\newcommand{\outcomeset}{\mathcal{D}}
\newcommand{\orderingset}{O}
\newcommand{\ordering}{\succ}
\newcommand{\numagents}{n}
\newcommand{\numsamples}{m}
\newcommand{\utility}{u}
\newcommand{\vertex}{v}
\newcommand{\vertexset}{V}

\newcommand{\edgeset}{E}
\newcommand{\timestep}{t}

\title{Scalable Mechanism Design for Multi-Agent Path Finding}
\author{
Paul Friedrich$^{*1,2}$
\and
Yulun Zhang$^{*3}$\and
Michael Curry$^{1,2,4}$\and
Ludwig Dierks$^{2,5}$\and
Stephen McAleer$^3$\and
Jiaoyang Li$^3$\and
Tuomas Sandholm$^{3,6}$\And
Sven Seuken$^{1,2}$\\
\affiliations
$^1$ETH AI Center\\
$^2$University of Zurich\\
$^3$Carnegie Mellon University\\
$^4$Harvard University\\
$^5$University of Illinois at Chicago\\
$^6$Optimized Markets, Strategy Robot, Strategic Machine
\emails
paul.friedrich@uzh.ch, yulunzhang@cmu.edu
}

\begin{document}

\maketitle

\def\thefootnote{*}\footnotetext{These authors contributed equally.}\def\thefootnote{\arabic{footnote}}

\begin{abstract}

\emph{Multi-Agent Path Finding (MAPF)} involves determining paths for multiple agents to travel simultaneously and collision-free through a shared area toward given goal locations.
This problem is computationally complex, especially when dealing with large numbers of agents, as is common in realistic applications like autonomous vehicle coordination. Finding an optimal solution is often computationally infeasible, making the use of approximate, suboptimal algorithms essential. Adding to the complexity, agents might act in a self-interested and strategic way, possibly misrepresenting their goals to the MAPF algorithm if it benefits them. Although the field of mechanism design offers tools to align incentives, using these tools without careful consideration can fail when only having access to approximately optimal outcomes. 
In this work, we introduce the problem of scalable mechanism design for MAPF and propose three strategyproof mechanisms, two of which even use approximate MAPF algorithms. We test our mechanisms on realistic MAPF domains with problem sizes ranging from dozens to hundreds of agents. We find that they improve welfare beyond a simple baseline.\footnote{Code at \url{https://github.com/lunjohnzhang/MAPF-Mechanism}} 
\end{abstract}

\section{Introduction} 
There are numerous emerging domains of large-scale resource allocation problems, such as allocating road capacity to cars or 3D airspace to unmanned aerial vehicles (UAVs), where \emph{multi-agent path finding (MAPF)} is required to calculate a valid allocation. 
In realistically-sized instances of such problems, non-colliding paths must be calculated for large numbers of agents simultaneously. Such computationally challenging domains are increasingly of real-world importance. Future UAV traffic in urban areas from parcel deliveries alone is projected to require managing tens of thousands of concurrent flight paths~\cite{doole2020estimation}.

Most MAPF research assumes that the true agent \emph{preferences} (e.g., cost for moving on the map) are known to (or even set by) the allocating system and focuses on finding computationally efficient ways to obtain solutions that minimize an aggregated agent cost measure, called \emph{optimal} solutions.
However, if agents are \emph{self-interested} and independently report their preferences to the system, it may be beneficial for them to \emph{misreport}. For example, an agent may wish to exaggerate her cost of a delay to ensure that the algorithm favors her over an agent that reported a lower cost when resolving a potential conflict.
In this \emph{non-cooperative} setting, the system may not know the agents' true preferences since they could be misreported. In that case, there is no guarantee that its MAPF algorithm finds a solution that maximizes true social welfare.

The field of \textit{mechanism design} deals with exactly this problem: designing systems that efficiently allocate scarce resources while ensuring that they are \emph{individually rational} and \emph{strategyproof}. These desirable properties guarantee that participating in the mechanism always benefits each agent and that agents are incentivized to report their private information truthfully to the mechanism, respectively. 
The celebrated \emph{Vickrey-Clarke-Groves (VCG)} mechanism~\cite{Vickrey1961Counterspeculation,Clarke1971Multipart,Groves1973Incentives} achieves this goal. 
It chooses an allocation that perfectly maximizes welfare given agents' reported preferences and then charges side payments to each agent. Via a careful choice of payments, VCG ensures that there is no incentive for agents to lie about their preferences. 

However, mechanism design has so far either not been applied to large-scale routing problems or has bypassed path finding, e.g., via congestion pricing in road domains~\cite{beheshtian2020bringing}. 
The main reason is that even classic, cooperative MAPF problems are computationally extremely challenging. Finding an optimal solution is NP-hard \cite{yuStructureIntractabilityOptimal2013}.
The state-of-the-art MAPF algorithm producing optimal solutions is \emph{conflict-based search (CBS)}~\cite{SHARON201540}. CBS first generates a shortest path for each agent and then resolves conflicts (path collisions) iteratively, building a constraint tree in a best-first-search manner. Its worst-case runtime scales exponentially in the number of agents, and identifying why certain scenarios are computationally feasible to optimally solve while others are not is an open area of research~\cite{gordonRevisitingComplexityAnalysis2021}. 

Much of the MAPF literature therefore employs (bounded) suboptimal algorithms or even counts finding any feasible path assignment as a success.
While there exist a number of suboptimal MAPF algorithms that work very well in practice as long as agent preferences are fully known~\cite{barerSuboptimalVariantsConflictBased2014,liEECBSBoundedSuboptimalSearch2021}, most are unsuitable for a mechanism with self-interested agents.
This is because  VCG-based mechanisms rely on finding an optimal assignment -- if the assignment is only approximately optimal, the resulting mechanism is generally no longer strategyproof~\cite{sandholm2002emediator}.

In this paper, we show how we can design suboptimal but more scalable MAPF algorithms that are made strategyproof through VCG-based payments by ensuring that they choose optimally from some restricted, fixed set of outcomes. This is called the \textit{maximal-in-range (MIR)} property~\cite{nisan2007computationally}, which guarantees that charging VCG payments results in no incentive to lie.

The rest of this paper is organized as follows. After giving an overview of related work (Section \ref{SEC:RW}), we introduce our formulation of the non-cooperative MAPF problem and illustrate the approach of pairing an optimal MAPF algorithm with VCG payments to achieve strategyproofness (Section \ref{SEC:MODEL}). 

In Section \ref{SEC:MECHA}, we first show how this approach applied to the optimal CBS algorithm results in a welfare-maximizing and strategyproof mechanism, which we call the \emph{payment-CBS (PCBS)} mechanism.

As we are interested in significantly larger problems, we successively restrict the search space to increase scalability.  
We focus on the class of \emph{prioritized} MAPF algorithms \cite{silverCooperativePathfinding2005}.
Instead of searching for a jointly optimal allocation, these algorithms search through a space of \emph{priority orderings}. They plan paths for agents sequentially, such that each planned path avoids all paths that belong to agents with higher priority. They are fast and efficient in practice but have no (bounded) suboptimality guarantees, and one can construct MAPF instances that cannot be solved by them. 

With this restriction, we first introduce the \emph{exhaustive-PBS (EPBS)} mechanism, which uses an allocation rule based on an adapted version of \emph{priority-based search (PBS)}~\cite{MaAAAI19}. 
It constructs a priority ordering, that is, a strict partial order on the set of agents, by iteratively adding a binary priority relation and re-planning accordingly for pairs of agents that conflict. By exhaustively expanding the PBS search tree, EPBS fulfills the MIR property.
EPBS illustrates how a suboptimal MAPF algorithm can be made strategyproof using MIR and our variation of VCG payments, which can be calculated without additional computational cost.

In a second step, we further restrict the search space to a fixed number of randomly drawn \emph{total} priority orderings. The resulting mechanism, which we call \emph{Monte-Carlo prioritized planning (MCPP)}, chooses the welfare-maximizing allocation among those resulting from applying \emph{prioritized planning (PP)} \cite{erdmann1987multiple} to the randomly drawn orderings.
In addition to being strategyproof through the combination of the MIR property with our VCG-based payments, MCPP can trade off optimality and scalability by varying the number of sampled priority orderings.

In summary, we present three strategyproof path allocation mechanisms that do not rely on computationally costly combinatorial auctions and model self-interested agents in the standard MAPF setting. We prove that applying the maximal-in-range property enables the use of suboptimal MAPF algorithms in strategyproof mechanisms.
In experiments across large MAPF benchmarks, we show that such mechanisms show a significant computational speedup compared to solutions relying on optimal MAPF algorithms and outperform a na{\"i}ve strategyproof baseline in terms of agent welfare.

\section{Related Work}\label{SEC:RW}
Mechanism design is used in a wide range of domains where goods must be allocated to self-interested agents. 
Example domains which have seen high-profile implementations include electricity markets~\cite{Cramton2017Electricity}, spectrum auctions~\cite{Cramton1997FCC}, course allocation at universities~\cite{
budish2011combinatorial,othmanFindingApproximateCompetitive2010}, vaccine allocation \cite{castilloMarketDesignAccelerate2021} and large-scale sourcing auctions, for instance for transportation services~\cite{sandholm2013very}. Given some chosen allocation rule, depending on the rule's structure, it may be possible to design a payment rule such that every agent has a dominant strategy to report her preferences truthfully, no matter what the other agents report~\cite{Myerson1981Optimal}.

Mechanism design where the allocated goods are paths usually fails to capture important characteristics of the MAPF domain. Ridesharing \cite{maSpatioTemporalPricingRidesharing2022} matches riders looking to go from A to B with drivers and aligns incentives for both sides with payments. Road pricing~\cite{beheshtian2020bringing} facilitates the optimally efficient movement of agents through a transportation network by charging higher payments for travel in areas with high congestion. Both approaches leave the decision on what route to take up to individual agents, meaning that collisions are not considered and cannot be prevented. \emph{UAV traffic management (UTM)} \cite{seukenMarketDesignDrone2022} is an emerging domain that combines heterogeneous, self-interested agents with explicit path allocations. Recent work in UTM focuses on resolving conflicts between agents, but offloads computational cost to the agent by requiring her to produce and submit her intended movement paths to the mechanism beforehand \cite{duchampAirTrafficDeconfliction2019}.

Prior research that explicitly models the MAPF domain has mostly focused on the \emph{cooperative} case \cite{sternMultiAgentPathfindingDefinitions2019}, with work on the \emph{non-cooperative} 
case just emerging and focusing on auction-based mechanisms. \citet{dasOMCoRPOnlineMechanism2021} assume a 2D grid topology with intersection vertices and solve the online problem by running separate VCG auctions at each intersection. \citet{amirMultiAgentPathfindingCombinatorial2015} model paths as \emph{bundles} of vertices, and use an iterative combinatorial auction mechanism called iBundle to sell non-colliding paths to self-interested agents in a way that achieves strategyproofness at the cost of poor scalability. \citet{gautierNegotiatedPathPlanning2022} build on their work with a suboptimal mechanism that achieves a computational speed-up at the cost of losing strategyproofness. \citet{chandra2023socialmapf}\footnote{While the paper claims strategyproofness, the relevant analysis is restricted to myopic agents only considering the next edge. Given that they employ sequential auctions for moving multiple edges, incentives do not translate to the full MAPF problem.} provide a planning mechanism that simplifies the strategy space by only having conflicting agents bid for movement priority in decentralized, repeated VCG auctions. We go a step further by reducing bid complexity for agents from valuing paths or priorities to merely providing values for a successful arrival and for their time. Furthermore, we improve on the high computational complexity of (VCG) auction-based approaches while still achieving their desired strategyproofness by using more scalable, established MAPF algorithms. Lastly, we use the classic MAPF formulation for the allocation problem and do not assume any additional topology on our graphs.
Research in prioritized-based methods for MAPF has used randomization to generate initial orderings that are then modified by the algorithm~\cite{BENNEWITZ200289}. However, unlike ours, this line of research does not consider incentives or produce strategyproof mechanisms.

\section{Problem Setup}\label{SEC:MODEL}
We consider the problem of assigning paths through a graph $(V,E)$ to $\numagents$ self-interested agents. For example, for UAV airspace assignment, each agent is an individual UAV operator, while the graph represents the assignable airspace.  

Each agent $i$ has a start vertex $\start_i$ and a goal vertex $\goal_i$ that she wants to reach as quickly as possible.  Time is discretized into timesteps indexed by $\timestep$. At each timestep, every agent can either move to an adjacent vertex or wait at its current vertex, both of which incurs \emph{cost} $\cost_i$. Additionally, each agent has some \emph{value} $\val_i$ for arriving at her goal vertex $\goal_i$. We also call the 4-tuple $\agenttype_i = (\start_i, \goal_i, \cost_i, \val_i)$ the agent's \emph{type}. A \emph{path} $\agentpath_i$ for agent $i$ is a sequence of vertices that are pairwise adjacent or identical. 
At the start, agents can wait at a private location (the \emph{garage vertex}) for a number of timesteps where they do not conflict with other agents and from which the only movement option is to their start vertex $\start_i$. A path begins at the garage or the agent's start vertex $\start_i$ and ends at her goal vertex $\goal_i$. Once an agent reaches her goal, she disappears, and her goal vertex becomes free for other agents. The cost of a path $c(\agentpath_i)$ is defined as $\cost_i |\agentpath_i|$. The \emph{welfare} of an agent represents her satisfaction with a path which she is assigned and is given by her value for arrival minus cost of traveling along the path, $\welfare_i(\agentpath_i) = \val_i - \cost_i(\agentpath_i)$. The sum of all agents' welfare is called \emph{social welfare}. To represent the welfare of all agents excluding $i$, we write $\welfare_{-i} = \sum_{j \neq i} \welfare_j(\agentpath_j).$
We further allow for the assignment of an empty path $\agentpath_i =\emptyset$ to any agent, which does not carry any value or cost.

Our theoretical results require the ability of our mechanisms to find feasible solutions for all instances and priority orderings given. Assuming either ``empty paths'' or ``garage vertex \& disappear-at-target'' is sufficient. The latter assumption makes the MAPF instance \emph{well-formed} \cite{capPrioritizedPlanningAlgorithms2015}, guaranteeing that prioritized algorithms (like EPBS and MCPP) are complete, i.e. always find feasible solutions. One of the two assumptions can be dropped at the cost of individual rationality or lower empirical scalability, respectively.

Our mechanisms can work with cost functions that use information other than $|\agentpath_i|$, as long as costs $\cost_i$ are independent of other agents' assignments. If there is such a dependence, all proposed algorithms stop being well-defined as they use single-agent best paths that no longer exist in isolation.

We define two types of conflict: A \emph{vertex conflict} $\langle i,j,\vertex,\timestep \rangle$ occurs when agents $i$ and $j$ attempt to occupy vertex $\vertex \in \vertexset$ at the same timestep $\timestep$; and an \emph{edge conflict} $\langle i, j,\utility,\vertex,\timestep \rangle$ occurs when agents $i$ and $j$ attempt to traverse the same edge $(u,v) \in \edgeset$ in opposite directions at the same time (between $\timestep-1$ and $\timestep$). 
A \emph{feasible assignment} is a set $\outcome = \lbrace \agentpath_1, \ldots, \agentpath_n \rbrace$ of conflict-free paths, one for
each agent. An \emph{optimal assignment} is the feasible assignment $\outcome^\ast$ that produces the maximum social welfare, $\outcome^\ast = \argmax_\outcome \sum_{i} \welfare_i(\agentpath_i^d)$.  
Thus we enhance the standard MAPF problem formulation, which aims to minimize \emph{flowtime} or \emph{makespan}, meaning the sum or the maximum of the arrival times of all agents at their goal vertices. 
Modeling all agents as having a value of zero and cost for time of one turns our social welfare objective into flowtime minimization, and further changing $\Sigma_i \welfare_i$ to $\min_i \welfare_i$ results in makespan minimization, both at no runtime cost.
Since both also remove the ability for agents to misreport and thus the need for strategyproofness, we omit their discussion.

In contrast to the classic MAPF problem, we assume that agents' types are private and not \emph{a priori} known to the mechanism that assigns the paths. Instead, each agent reports her type to the mechanism. We denote by $\reportedtype= (\reportedtype_1, \ldots, \reportedtype_n)$ a set of reported types that serve as the mechanism input, and by $\agentpath(\reportedtype)$ the path assignment that the mechanism outputs. 
We assume agents are self-interested and may misreport a type $\reportedtype_i\not=\agenttype_i$ in such a way as to ensure that the mechanism selects an outcome that is favorable to them. A (social) welfare optimizing mechanism, for instance, finds the outcome which maximizes the sum of \emph{reported agent welfare}. This outcome may be arbitrarily suboptimal for other individual agents (hurting egalitarian social welfare, that is, fairness) or the population as a whole (hurting social welfare, that is, efficiency). We assume that agents do not misreport their start and goal vertices, as they would have negative infinite welfare for receiving a path that does not take them from their start to their goal, reducing the agent's strategic choices to $\cost_i$ and $\val_i$.

To align incentives, we assume the mechanism may charge \emph{payment} $\payment_i(\reportedtype)$ to agent $i$ upon assigning them a path. Consequently, an agent's \emph{utility} is given by the welfare of her assignment minus the payment she is charged,
\begin{align*}
    \utility_i(\reportedtype) =  \val_i - \cost_i \bigl(\agentpath_i(\reportedtype)\bigr) - \payment_i(\reportedtype).
\end{align*}
In a slight abuse of notation, we also write $\utility_i(\reportedtype_i,\reportedtype_{-i})$ for the utility resulting from agent $i$ reporting $\reportedtype_i$ if the set of reports from the remaining agents is $\reportedtype_{-i}$. 
We assume that given any $\reportedtype_{-i}$, any agent $i$ will report the type $\reportedtype_i$ that maximizes her utility $\utility_i(\reportedtype_i,\reportedtype_{-i})$. 
To facilitate our analysis,  we restrict our attention to \emph{strategyproof (SP)} mechanisms, where it is a dominant strategy for agents to report their true type $\agenttype_i$, independent of the reports of the other agents $\reportedtype_{-i}$.

Additionally, we focus on \emph{individually rational (IR)} mechanisms, that is, mechanisms that guarantee every agent (weakly) positive utility $\utility_i(\reportedtype)\geq0$. Such mechanisms guarantee that agents are incentivized to participate in the mechanism, as they cannot lose utility doing so~\cite{Krishna2009Auction}.

\paragraph{Summary}
    We aim to find a mechanism that, given a set of reports, outputs an assignment with a valid path (or no path) for each agent, such that it \emph{maximizes welfare} and that the mechanism is \emph{strategyproof} and \emph{individually rational}.

\subsection{Mechanism Design Desiderata for MAPF}
We now explore how to achieve two key desiderata: strategyproofness and individual rationality. 

For mechanisms such as our payment-CBS (PCBS) which find optimal allocations,  strategyproofness can be guaranteed by charging agents VCG payments. However, for suboptimal mechanisms such as our exhaustive-PBS (EPBS) and Monte-Carlo prioritized planning (MCPP), an additional property called \emph{maximal-in-range} has to be satisfied. We show that this property is sufficient to guarantee strategyproofness when combined with a variation on VCG payments, which we call \emph{VCG-based payments.}

\begin{definition}  
    An allocation rule $\allocrule$ satisfies the \emph{maximal-in-range property (MIR) with range $\range$}, if there exists a fixed subset $\range$ of all possible allocations $\outcomeset$ such that for all possible reports of misreportable information $\reportedtype=(\reportedtype_1, \ldots, \reportedtype_\numagents)$, the allocation rule chooses the outcome in the range $\range$ which results in the highest reported welfare \emph{(``the outcome which maximizes reported welfare over $\range$'')}. 
    In other words, iff
    $$\exists \range \subseteq \outcomeset: \forall \reportedtype: \allocrule(\reportedtype) = \argmax_{\outcome \in S} \sum_{i=1}^\numagents \reportedwelfare_i(\outcome).$$
\end{definition}

Our EPBS and MCPP mechanisms use VCG payments with one variation. Let $\outcome^\ast$ be the assignment which maximizes social welfare within the mechanism's allocation rule's range $\range$, that is, ${\outcome^\ast := \argmax_{\outcome \in S} \sum_{i} \reportedval_i - \reportedcost_i(\agentpath_i^\outcome)}$. The VCG approach also requires calculating the \emph{counterfactual} assignment $\outcome^\ast_{-i}$ which maximizes social welfare in the hypothetical problem instance where agent $i$ is not present. 

\begin{definition}
    The \emph{VCG-based payment rule} charges to agent $i$ the difference in total reported welfare of all agents excluding $i$ of the factual and counterfactual welfare-maximizing assignment: 
$\payment_i(\reportedtype):= \reportedwelfare_{-i}(\outcome^*_{-i})-\reportedwelfare_{-i}(\outcome^*).$
\end{definition}
Intuitively, the payment represents the amount by which agent $i$'s presence has caused a worse assignment and lowered social welfare for all other agents, also called agent $i$'s \emph{externality}. Typically, calculating all $\outcome^\ast_{-i}$ requires solving the underlying optimization problem once per agent for an altered problem instance where that agent is not present. This is extremely costly in the MAPF domain. Our EPBS and MCPP mechanisms use a slightly different $\outcome^\ast_{-i}$, which sidesteps this requirement and ensures that payments are non-negative. They take all assignments within their range $\range$, set agent $i$'s value and cost to 0 but leave the assignment's paths unchanged (i.e., agent $i$'s path is still present), and select the so altered assignment with the highest welfare. Equivalently, ${\outcome^\ast_{-i} = \argmax_{\outcome \in S} \sum_{j \neq i} \reportedval_j - \reportedcost_j(\agentpath_j^\outcome).}$ All counterfactual assignments are thus drawn from the range that was already explored when finding $\outcome^\ast$, avoiding costly recalculation.

To ensure \emph{individual rationality}, or non-negative utilities for all agents, our mechanisms will not assign paths to agents if the path's cost exceeds the agent's value. If an agent is assigned a path whose cost is greater than the value the agent would gain from completing it, i.e., if $\cost_i |\agentpath_i| > \val_i$, we cap the agent's cost for that path at $\val_i$ and assume that agent $i$ will choose not to move at all. Formally, we define $\cost_i(\agentpath_i) := \min(\val_i,\cost_i|\agentpath_i|)$ such that an agent's welfare becomes $\welfare_i := \val_i - \cost_i(\agentpath_i)=\max(0,\val_i - \cost_i|\agentpath_i|)$.
If an agent is assigned a path for which she would have negative welfare and decides not to take that path due to individual rationality, our adjustment sets her welfare to zero.\footnote{MCPP and EPBS treat empty, 0-welfare paths as space unavailable to lower priority agents, a necessary compromise to avoid positive externalities and enable computational cheapness of payments.} Then 
$$\outcome^\ast = \argmax_{\outcome\in\range} \sum_i \reportedwelfare_i(\outcome)=\argmax_{\outcome\in\range} \sum_{j\neq i} \reportedwelfare_i(\outcome) = \outcome_{-i}^\ast,$$
which implies that $\payment_i = 0$, and since $\welfare_i = 0$ also $\utility_i = 0$.

\begin{restatable}{lemma}{mir}\label{thm:mir}
    Let $(\allocrule,\payment)$ be a mechanism consisting of a path allocation rule $\allocrule$ that satisfies MIR with range $\range$ and the VCG-based payment rule $\payment$. Then $(\allocrule,\payment)$ is strategyproof.
\end{restatable}

\noindent The proof is in the appendix and follows \citeauthor{nisan2007computationally}'s (2007) argumentation using our domain-specific adaptations.

\begin{restatable}{lemma}{irnegpayments}\label{thm:ir-negpayments}
    Let $(\allocrule,\payment)$ be a mechanism consisting of a path allocation rule $\allocrule$ and the VCG-based payment rule $\payment$. Let $c_i(\agentpath_i) = \max(\val_i, \cost_i|\agentpath_i|)$ as above. Then
    \begin{enumerate}[label=(\roman*)]
        \item The mechanism $(\allocrule,\payment)$ is \emph{individually rational}.
        \item For all agents $i$ and reported types $\reportedtype$, $\payment_i(\reportedtype) \geq 0$.\\ That is, the mechanism has \emph{no negative payments}.
    \end{enumerate}
\end{restatable}

\noindent The proof is in the appendix.

A welfare-maximizing MAPF algorithm fulfills MIR as long as two conditions are met. First, an agent cannot change the range of assignments explored by the algorithm by misreporting. Second, the algorithm breaks ties between equal-length paths in a way that is independent of the agents' reports (which we assume). Together, they ensure that the mechanism constructs the same range for all possible agent reports.

Standard suboptimal MAPF algorithms like EECBS \cite{liEECBSBoundedSuboptimalSearch2021} may be a natural first target in search of scalability. Yet, none, to our knowledge, fulfill the MIR property, as their search tree expansion makes use of heuristics that 
agents can influence by misreporting their cost. 
Both of our suboptimal mechanisms use priority-based allocation rules. We show that this class of MAPF algorithms can be altered in natural ways to achieve MIR and thus strategyproofness.

\section{Mechanisms}\label{SEC:MECHA}
We present three mechanisms that solve the classical MAPF problem formulation for self-interested agents. All three fulfill the mechanism design desiderata of \emph{strategyproofness}, \emph{individual rationality} and \emph{no negative payments}.

\subsection{Payment-CBS (PCBS)}
PCBS constitutes our optimal benchmark mechanism. Given reported agent types, it first finds a reported social welfare maximizing allocation $\outcome^\ast \in \outcomeset$ using \emph{conflict-based search (CBS)}~\cite{SHARON201540} and then charges agents VCG payments. As it fulfills the definition of a VCG mechanism, it is \emph{strategyproof} (see Proof of \Cref{thm:mir}), and by \Cref{thm:ir-negpayments} it is \emph{individually rational} and charges \emph{no negative payments}. 

While EPBS and MCPP compute all outcomes within their range before selecting the welfare-maximal one, PCBS does not. To find the VCG payment for an agent, CBS must be run on the setting excluding that agent in order to obtain the counterfactual $\outcome_{-i}^\ast$. In total, the mechanism runs CBS once to find the optimal allocation $\outcome^\ast$ and again for each $i=1,\ldots,\numagents$ to compute $\outcome_{-i}^\ast$ and thus $\payment_i$. However, these runs could occur in parallel, reducing PCBS's real runtime to that of regular CBS' (given sufficiently many CPUs). While it thus shares a level of limited scalability with the main optimal, strategyproof benchmark of iBundle \cite{amir2014empirical}, PCBS eliminates iBundle's prohibitive communication complexity.

\subsection{Exhaustive-PBS (EPBS)}
Our EPBS mechanism, derived from \emph{priority-based search (PBS)}~\cite{MaAAAI19}, illustrates how a suboptimal MAPF algorithm can be adapted to fulfill the MIR property. Using a two-stage algorithm it constructs a \emph{priority ordering}, i.e. a strict partial order on the set of agents. Starting from an empty priority ordering and a set of individually optimal but colliding paths, it builds a search tree where each node contains a priority ordering and a movement plan with its cost and list of collisions. It iteratively expands a node, picks a collision between two agents and resolves it by expanding into two child nodes, which each add one of the two possible priority relations between the two agents to its priority ordering, ensuring that the priority ordering stays transitive. Paths are re-planned using the new priority ordering with a single-agent low-level search such as $A^\ast$, collisions are tracked, welfare is calculated, and the next node is expanded. Once all nodes are expanded to collision-free leaves, EPBS terminates and returns the highest welfare leaf node's allocation. In contrast, PBS uses depth-first search, expanding the highest-welfare child and terminating once a single node is collision-free.

To achieve welfare-maximizing allocations with certainty, we would in principle need to search through, and construct assignments for, all possible priority orderings. Using the MIR property, we reduce the size of our search space to just the leaves of EPBS's search tree. The search tree's depth is $\mathcal{O}(n^2)$, as each node expansion adds one pair of agents to the priority ordering and no branch can add the same pair twice.

While finding the allocation $\outcome^\ast$ in the worst case thus scales exponentially with the number of agents, by fully expanding its search tree EPBS calculates all outcomes within its range in one go. Due to the variation we make in defining VCG-based payments, the range includes all counterfactuals $\outcome_{-i}^\ast$. Thus, to run the mechanism, meaning to find $\outcome^\ast$ and calculate all payments $\payment_i$, the search tree needs to be constructed only once.\footnote{Here we touch on a fundamental question in \emph{algorithmic mechanism design} posed by \cite{Nisan2001Algorithmic}: How often does one need to solve the underlying optimization problem to obtain $\numagents$ agents' VCG payments? \citet{hershbergerVickreyPricesShortest2001} prove that the answer is one in the domain of finding a shortest path in a network where edges with costs represent self-interested agents.} Building this tree could be parallelized to some degree~\cite{swiechowski2023monte} (although we have not done so in this work), as each new child node's path replanning task based on its unique ordering requires only the information in its parent node and does not need to access a data structure that is shared with nodes other than its own future children.

\begin{proposition}
        EPBS is strategyproof.
\end{proposition} 
\begin{proof}
    EPBS uses a lexicographic ordering of agents (which we assume to be fixed) and their reports of start- and goal vertices (which we assume to not be misreportable) for exploring assignments. Given a set of agent types that vary only in misreportable information, namely costs $\cost_i$ and values $\val_i$, EPBS constructs the same search tree from each set, resulting in the same set of leaf nodes that contain collision-free assignments $\outcomeset_{\text{PBS}}$. Therefore this set of assignments fulfills the definition of EPBS's \emph{range}. Since EPBS chooses the welfare maximizing assignment over $\outcomeset_{\text{PBS}}$, it fulfills \emph{MIR} and is thus made strategyproof by its payments according to \Cref{thm:mir}.
\end{proof}

\begin{corollary}
    EPBS is individually rational and has non-negative payments.
\end{corollary}

\noindent The proof follows directly from \Cref{thm:ir-negpayments}.

\subsection{Monte-Carlo Prioritized Planning (MCPP)}
MCPP demonstrates how MIR allows for creating strategyproof MAPF mechanisms that scale to the theoretical optimum, at the cost of suboptimality.
Our solution utilizes \emph{prioritized planning (PP)}~\cite{erdmann1987multiple}. In contrast to (E)PBS, PP takes as an input a \emph{total priority ordering} $\ordering$ and expands it into a set of conflict-free paths by sequentially assigning each agent (starting with the highest priority one) the shortest possible path that avoids all agents of higher priority. In this way, calculating one assignment takes only $\numagents$ low-level path-finding searches. 

Our MCPP mechanism (see \Cref{mech:mcpp}) samples a subset $\orderingset$ of all possible total priority orderings in a way that is independent of agent reports, then runs prioritized planning on each ordering in $\orderingset$ yielding a range of path assignments $\outcomeset_\orderingset$, selects the welfare-maximizing assignment $\outcome^\ast \in \outcomeset_\orderingset$ and charges VCG-based payments. Concretely, MCPP randomly samples $\numsamples$ total priority orderings from the space of $\numagents!$ possible ones. We can directly trade off computational efficiency and optimality by changing the number of samples $\numsamples$. Increasing $\numsamples$ increases the probability of $\orderingset$ containing orderings that generate near-optimal assignments, and thus the expected welfare of MCPP. We avoid the usually costly computation of counterfactual outcomes for the VCG payments, as the counterfactual $\outcome_{-i}^\ast$ is selected from the range $\outcomeset_\orderingset$, which is already fully explored during the finding of $\outcome^\ast$. The computational complexity of MCPP therefore scales linearly with the number of agents and the number of samples $\numsamples$ that are selected, a drastic improvement over the exponential scaling of PCBS, EPBS and auction-based strategyproof MAPF mechanisms. 
As the sampled orderings are independent, their $\numsamples$ outcomes could be calculated in parallel, resulting in a real runtime for MCPP of just $\numagents$ calls of $A^\ast$.

\begin{proposition}
        MCPP is strategyproof.
\end{proposition} 
\begin{proof}
    MCPP generates a subset $\orderingset$ of priority orderings without considering agent reports. Since each ordering $\ordering$ in $\orderingset$ corresponds to a unique path assignment $\outcome_\ordering$, the set $\outcomeset_\orderingset = \{\outcome_\ordering | \ordering \in \orderingset\}$ fulfills the definition of MCPP's \emph{range}. As MCPP chooses the welfare-maximising assignment over $\outcomeset_\ordering$, it fulfills \emph{MIR} and is thus made strategyproof by its payments according to \Cref{thm:mir}.
\end{proof}

\begin{corollary}
    MCPP is individually rational and has non-negative payments.
\end{corollary}

\noindent The proof follows directly from \Cref{thm:ir-negpayments}.

\begin{remark}
    If there is public knowledge about agents, we can also incorporate it. For example, if it is known that an agent generally has high costs across all problem instances, we can improve the expected welfare of the mechanism by restricting $\orderingset$ to total orders that assign that agent a high priority.
\end{remark}

\begin{algorithm}
\caption{\textsc{MCPP}}\label{mech:mcpp}
\DontPrintSemicolon
\KwIn{Agent reports $(\agenttype_1,\ldots,\agenttype_\numagents)$, No. of samples $m$}
\KwOut{Conflict-free paths $(\agentpath_1, \ldots, \agentpath_\numagents)$ and payments $(\payment_1, \ldots, \payment_\numagents)$}
\BlankLine

\textbf{Allocation rule: Prioritized Planning}\;
\nl Draw a set $\orderingset$ of $\numsamples$ total priority orderings $\ordering$ independently from agent reports.\;
\nl \For{$\ordering \in \orderingset$}{
$\outcome_\ordering \leftarrow PP(\ordering)$}
Result: set of assignments $\{\outcome_\ordering | \ordering \in \orderingset\} = \outcomeset_\orderingset$, each assignment contains conflict-free paths $(\agentpath_1, \ldots, \agentpath_\numagents)$\;
\nl $\welfare(\outcome) \leftarrow \sum_i \bigl(\val_i-\cost(\agentpath_i^{\outcome})\bigr)$ for all $\outcome \in \outcomeset_\orderingset$\;
\nl $\outcome^\ast \leftarrow \argmax_{\outcome\in\outcomeset_\orderingset}\welfare(\outcome)$, the welfare maximising assignment over $\outcomeset_\orderingset$\;
\BlankLine
\textbf{Payment rule: VCG-based payments}\;

\nl \For{agents $i=1,\ldots,\numagents$}{
    $\welfare_{-i}(\outcome) \leftarrow \sum_{j\neq i} \bigl(\val_j - \cost(\agentpath_j^\outcome)\bigr)$ for all $\outcome \in \outcomeset_\orderingset$ \;
    $\outcome^*_{-i} \leftarrow \argmax_{\outcome \in \outcomeset_\orderingset} \welfare_{-i}(\outcome)$\;
    $\payment_i \leftarrow \welfare_{-i}(\outcome^*_{-i})-\welfare_{-i}(\outcome^*)$\;
}
\nl Assign each agent $i$ the path $\agentpath_i \in d^\ast$\;
\nl Charge each agent $i$ the payment $\payment_i$\;
\end{algorithm}
\section{Experiments}
\begin{figure*}[ht]
    \centering
    \includegraphics[width=0.85\textwidth]{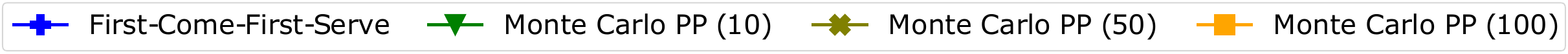}
        \begin{subfigure}{0.3\textwidth}
        \centering
        \offinterlineskip
        \includegraphics[width=1\textwidth]{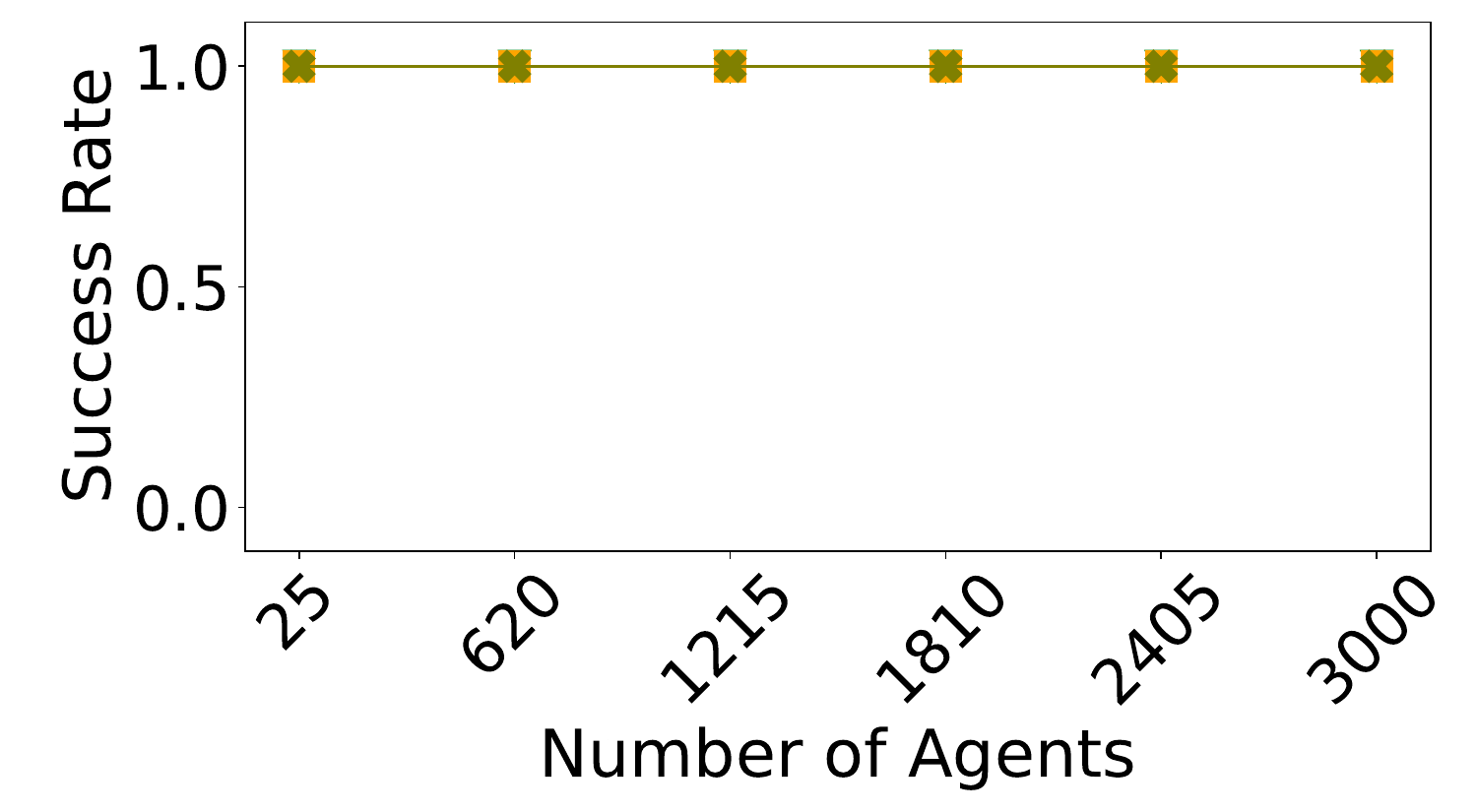}
        \label{fig:mcpp-compare-success}
    \end{subfigure}
    \hfill
    \begin{subfigure}{0.3\textwidth}
        \centering
        \offinterlineskip
        \includegraphics[width=1\textwidth]{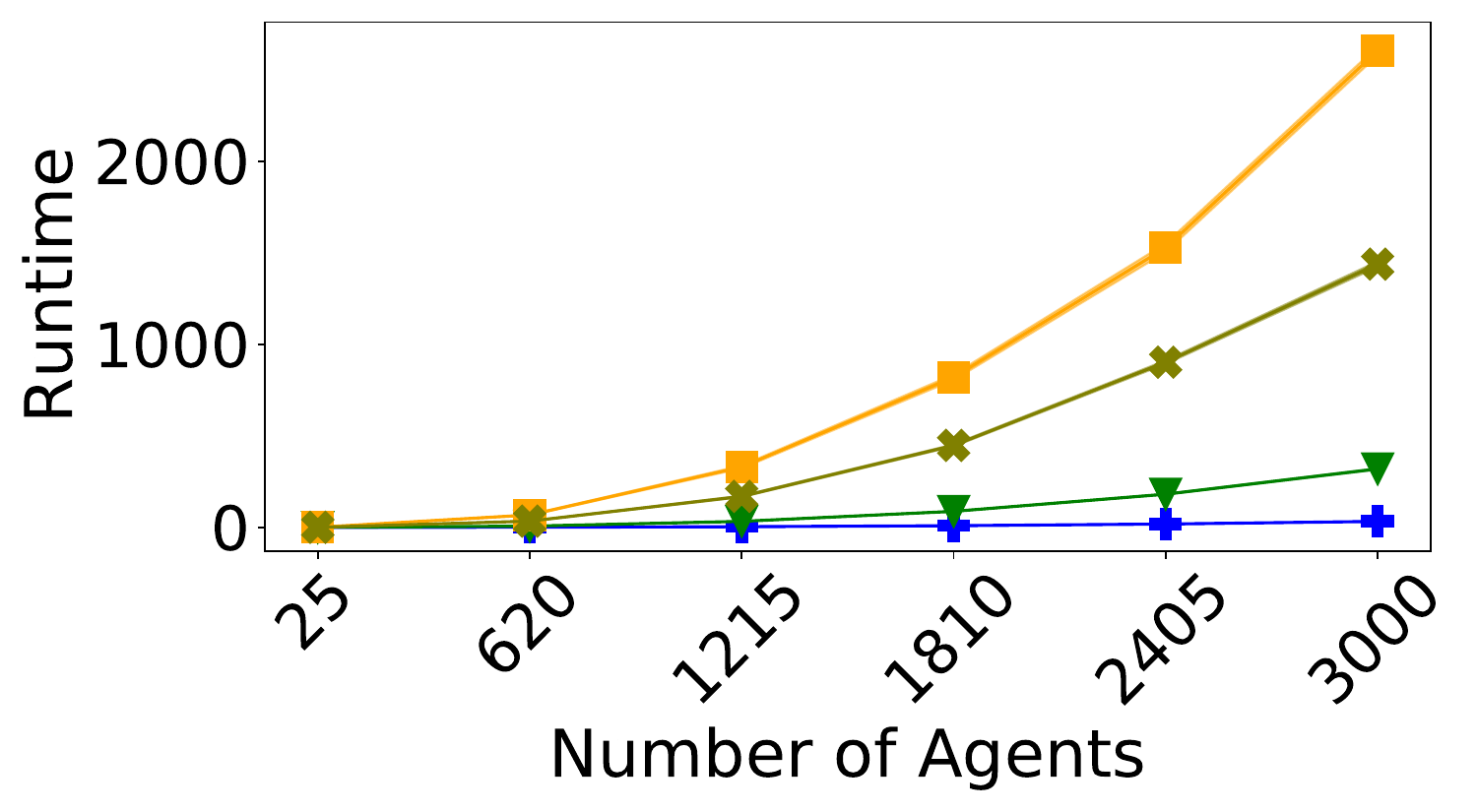}
        \label{fig:mcpp-compare-runtime}
    \end{subfigure}
    \hfill
    \begin{subfigure}{0.3\textwidth}
        \centering
        \offinterlineskip
        \includegraphics[width=1\textwidth]{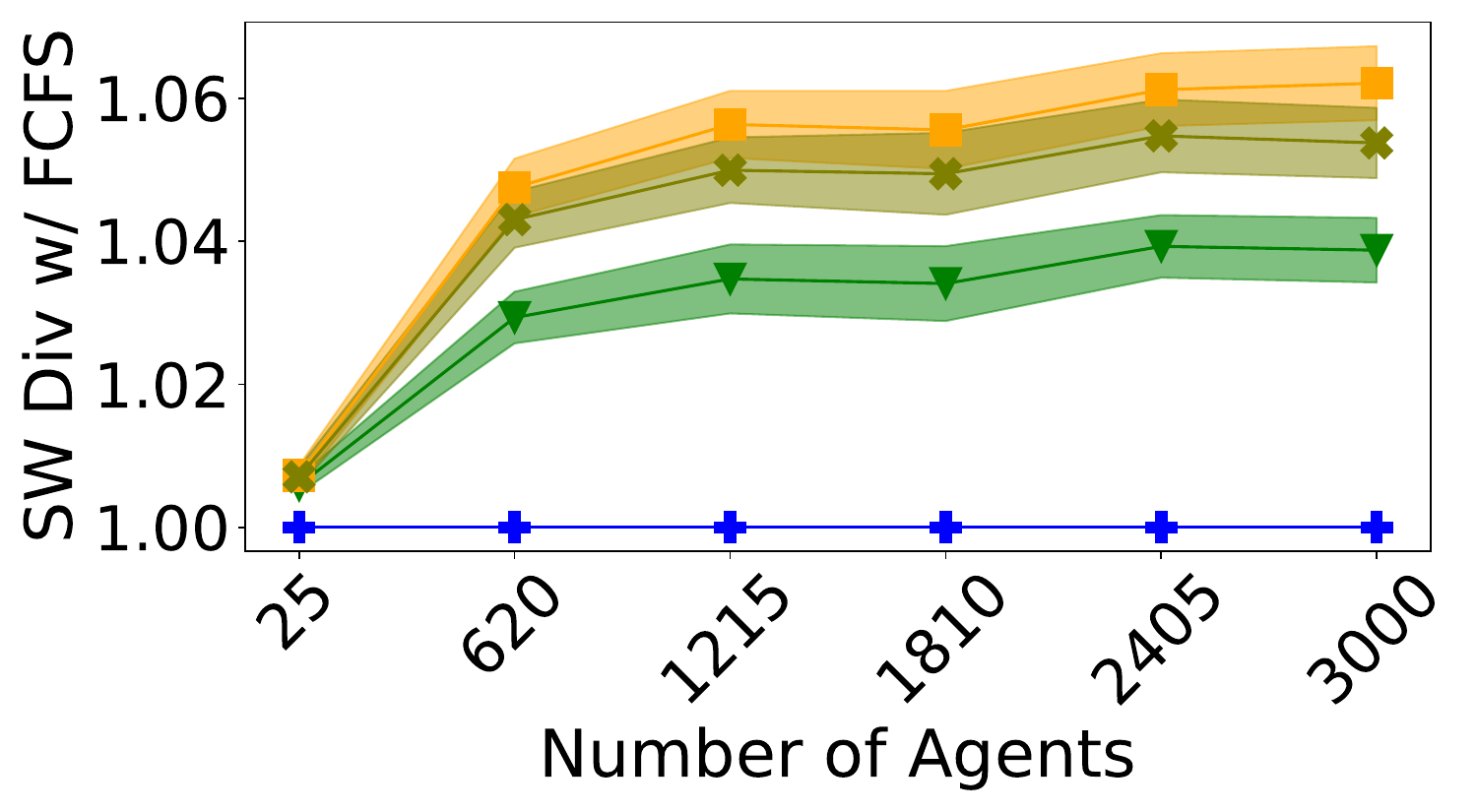}
        \label{fig:mcpp-compare-welfare-subopt}
    \end{subfigure}
    \caption{Success rate, runtime, and ratio-to-baseline of social welfare (SW) of Monte Carlo PP with $m \in \{10, 50, 100\}$ samples.}
    \label{fig:compare-mcpp}
\end{figure*}

Our experiments verify that our mechanisms' behavior follows our claims on optimality, scalability, and payments.
Our benchmarks are strategyproof mechanisms whose allocation rule solves classical MAPF, iBundle \cite{amirMultiAgentPathfindingCombinatorial2015} being the most advanced competition. We could not obtain its reference implementation to be able to accurately compare runtimes. In the original benchmarks \cite{amir2014empirical}, iBundle shows similar scalability as CBS. As our PCBS, if parallelized, has identical time complexity as CBS, the comparison can be extrapolated.
We also do not compare scalability against SOTA approximate MAPF algorithms, as they are not strategyproof and feature structures unsuited for agent cost-weighted social welfare objectives.

\subsection{Experiment Setup}\label{SEC:EXP}

We evaluate payment-CBS (PCBS), exhaustive-PBS (EPBS) and Monte-Carlo prioritized planning (MCPP) on four 2D maps selected from the commonly used MAPF benchmarks by~\citet{sternMultiAgentPathfindingDefinitions2019}. The appendix contains evaluations on derived 3D maps. 
We include the so-called First-Come-First-Serve (FCFS) allocation rule, which runs prioritized planning on a single randomly sampled priority ordering as a lower bound for achieved welfare and runtime, and does not include payments. While suboptimal, it is the best pre-existing mechanism that scales to larger instances and is used in practice, e.g., in the EU's regulation on assigning airspace to drone traffic~\cite{theeuropeancommissionCommissionImplementingRegulation2021}.
We independently sample each agent's value from $\mathcal{U}([0,1])$ and cost from $\mathcal{U}([0, \val_i / \text{dist}(\start_i, \goal_i)])$, where $\text{dist}(\start_i, \goal_i)$ is the length of the shortest path from $\start_i$ to $\goal_i$ in isolation. This ensures that near-optimal paths are likely to result in positive welfare. MCPP uses a sample size of $\numsamples=100$, except in Figure \ref{fig:compare-mcpp}.
Each run uses a set of $100$ instances. We set a runtime limit for all mechanisms at $3600$ seconds for the \emph{random-32-32-20} and \emph{dens312d} maps and $5400$ seconds for the \emph{dens520d} and \emph{Paris-1-256} maps. Success rate refers to the fraction of runs that finish within the time limit. The runtime for individual runs that do not finish within the limit is set to the limit.
We did not use parallelized implementations of our mechanisms. Our code was implemented in C++ and ran on a 2x12-core Intel Xeon E5-2650v4 machine with 128 GB of RAM. 

\subsection{Results}

\begin{figure*}[!t]
    \centering
    \offinterlineskip
    \includegraphics[width=0.85\textwidth]{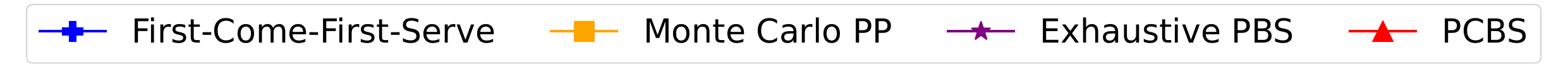}

    \label{fig:legend}
    \begin{subfigure}{0.246\textwidth}
        \centering
        \offinterlineskip
        \includegraphics[width=1\textwidth, height=0.6\textwidth]{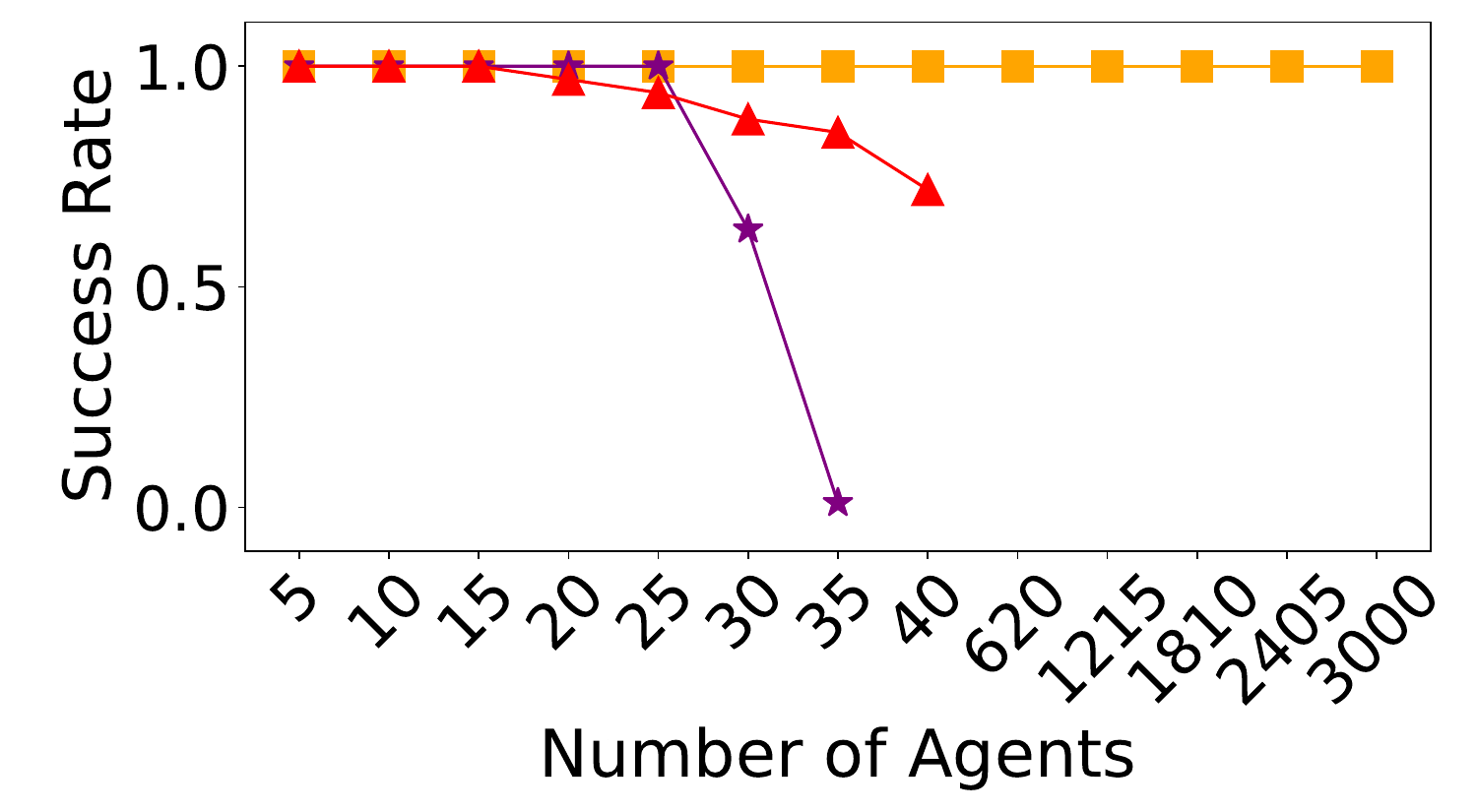}
        \label{fig:random-32-32-20-layer=1-success}
    \end{subfigure}
    \hfill
    \begin{subfigure}{0.246\textwidth}
        \centering
        \offinterlineskip
        \includegraphics[width=1\textwidth, height=0.6\textwidth]{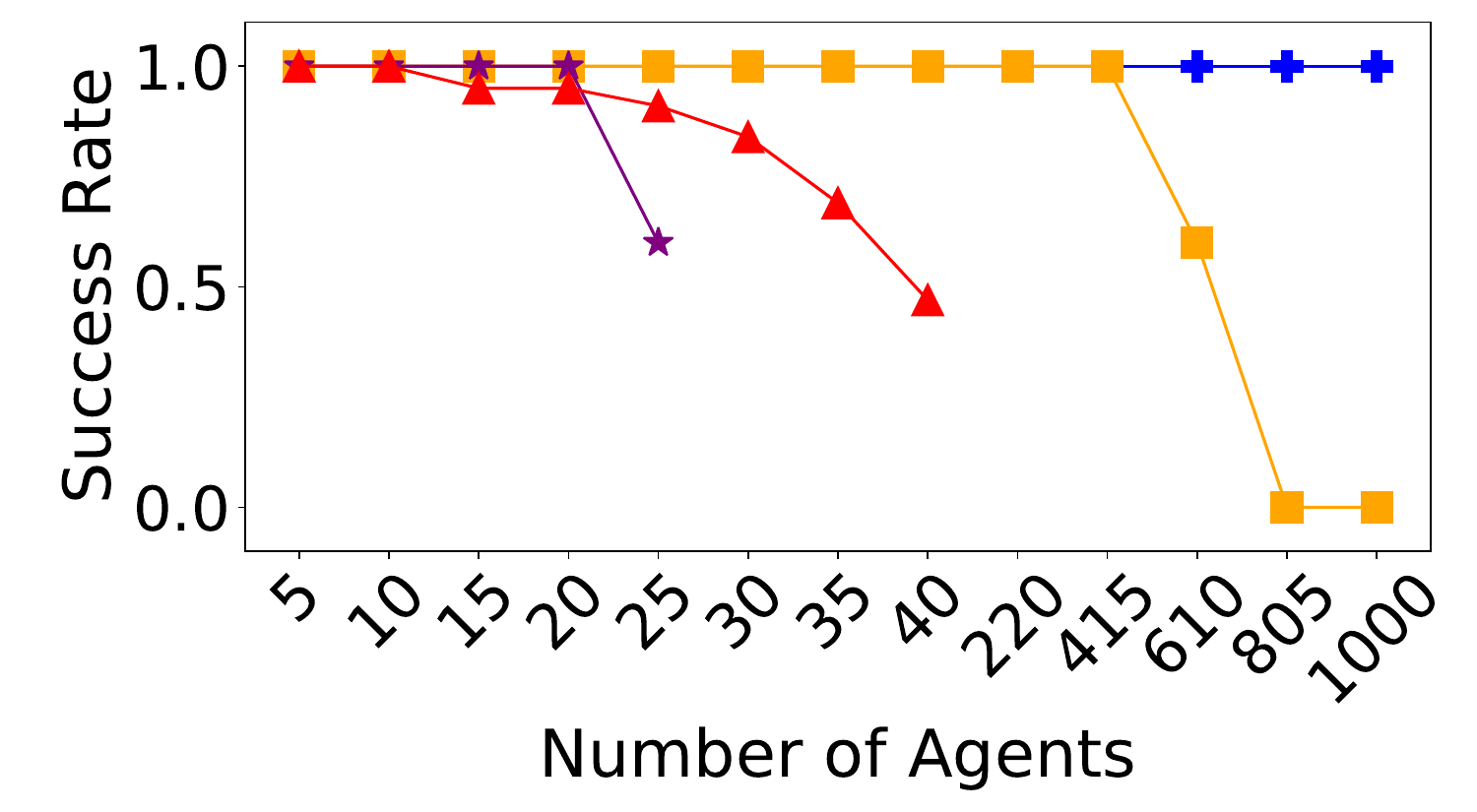}
        \label{fig:den312d-layer=1-success}
    \end{subfigure}
    \hfill
    \begin{subfigure}{0.246\textwidth}
        \centering
        \offinterlineskip
        \includegraphics[width=1\textwidth, height=0.6\textwidth]{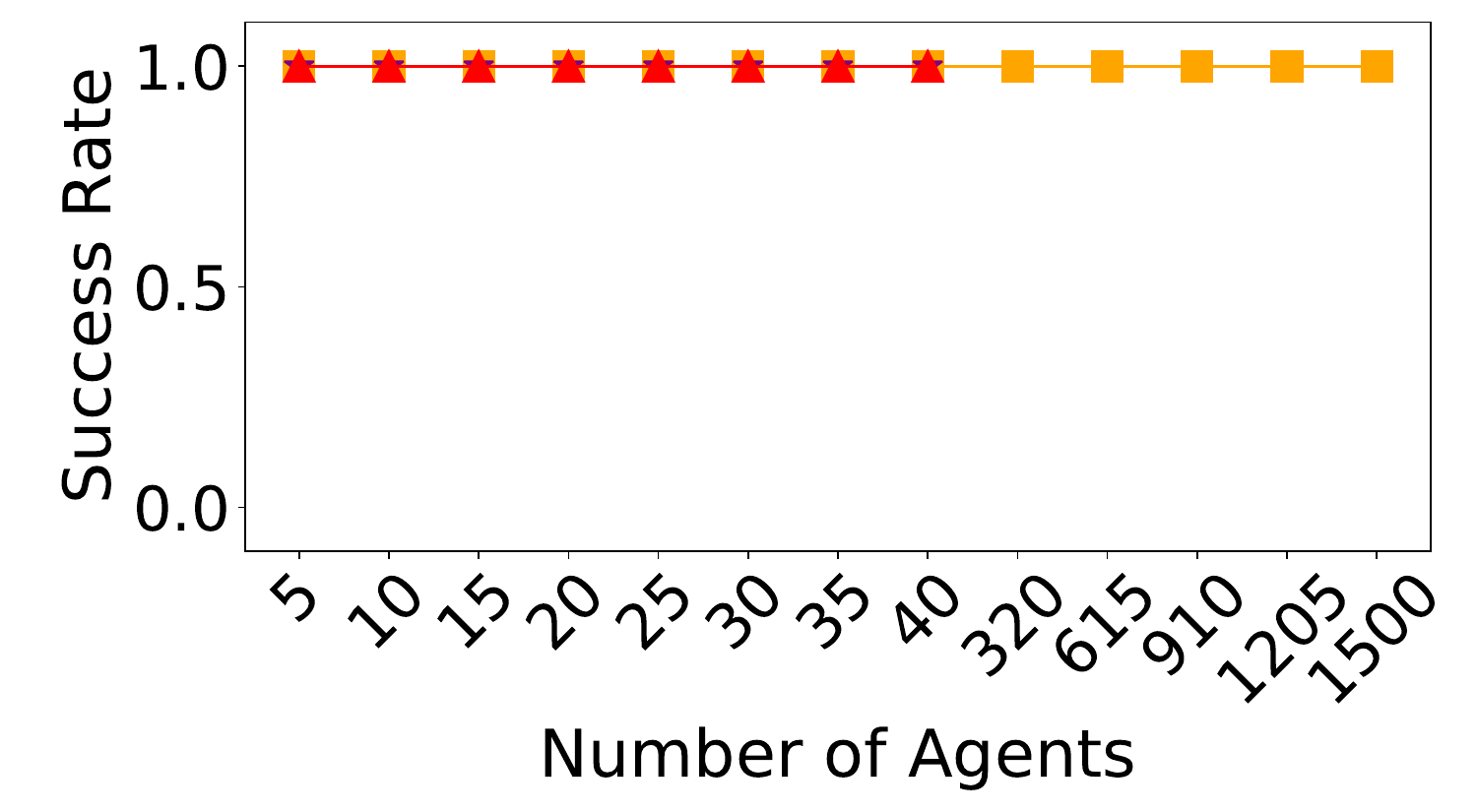}
        \label{fig:Paris-1-256-layer=1-success}
    \end{subfigure}
    \hfill
    \begin{subfigure}{0.246\textwidth}
        \centering
        \offinterlineskip
        \includegraphics[width=1\textwidth, height=0.6\textwidth]{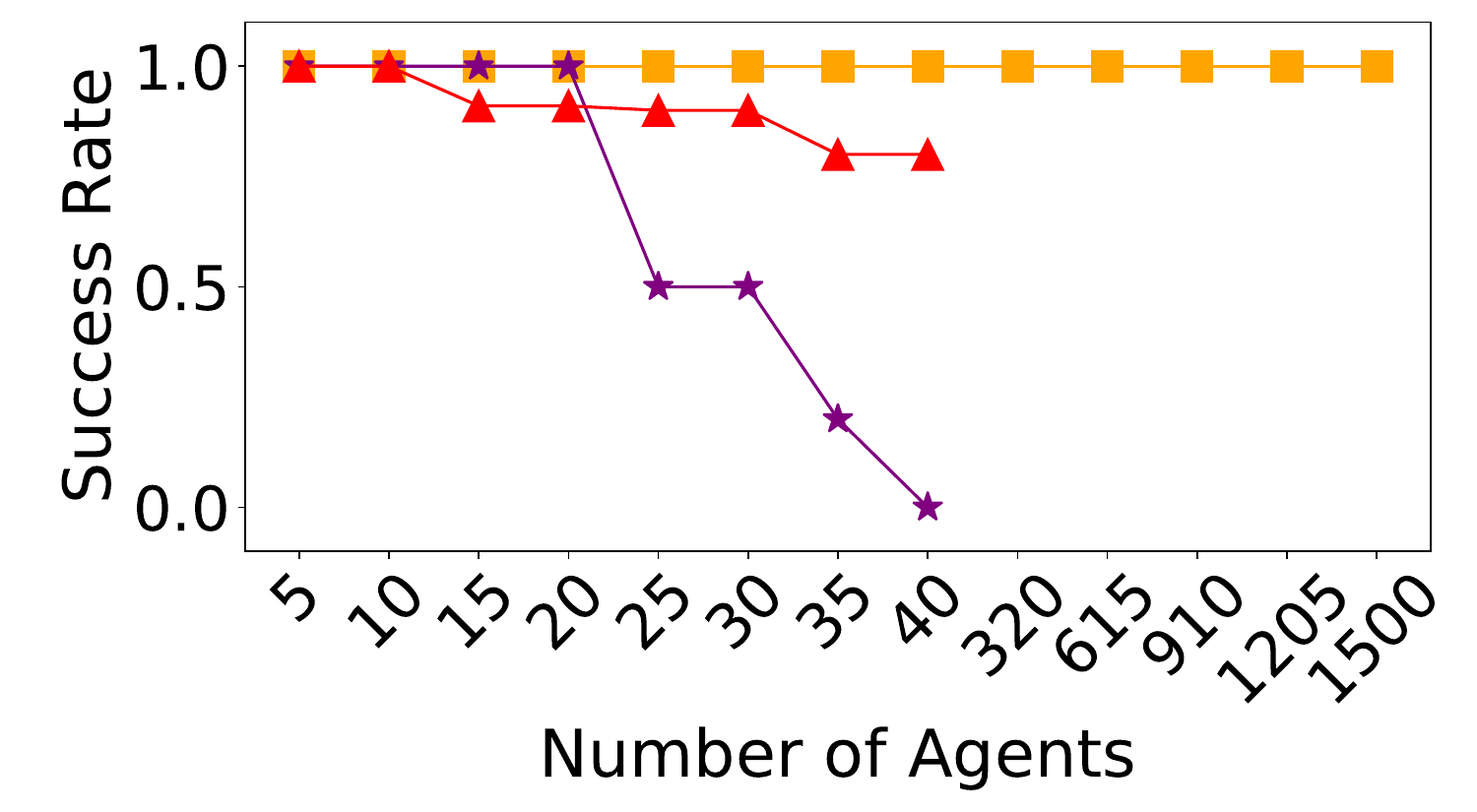}
        \label{fig:den520d-layer=1-success}
    \end{subfigure}
    \\
    \begin{subfigure}{0.246\textwidth}
        \centering
        \offinterlineskip
        \includegraphics[width=1\textwidth, height=0.6\textwidth]{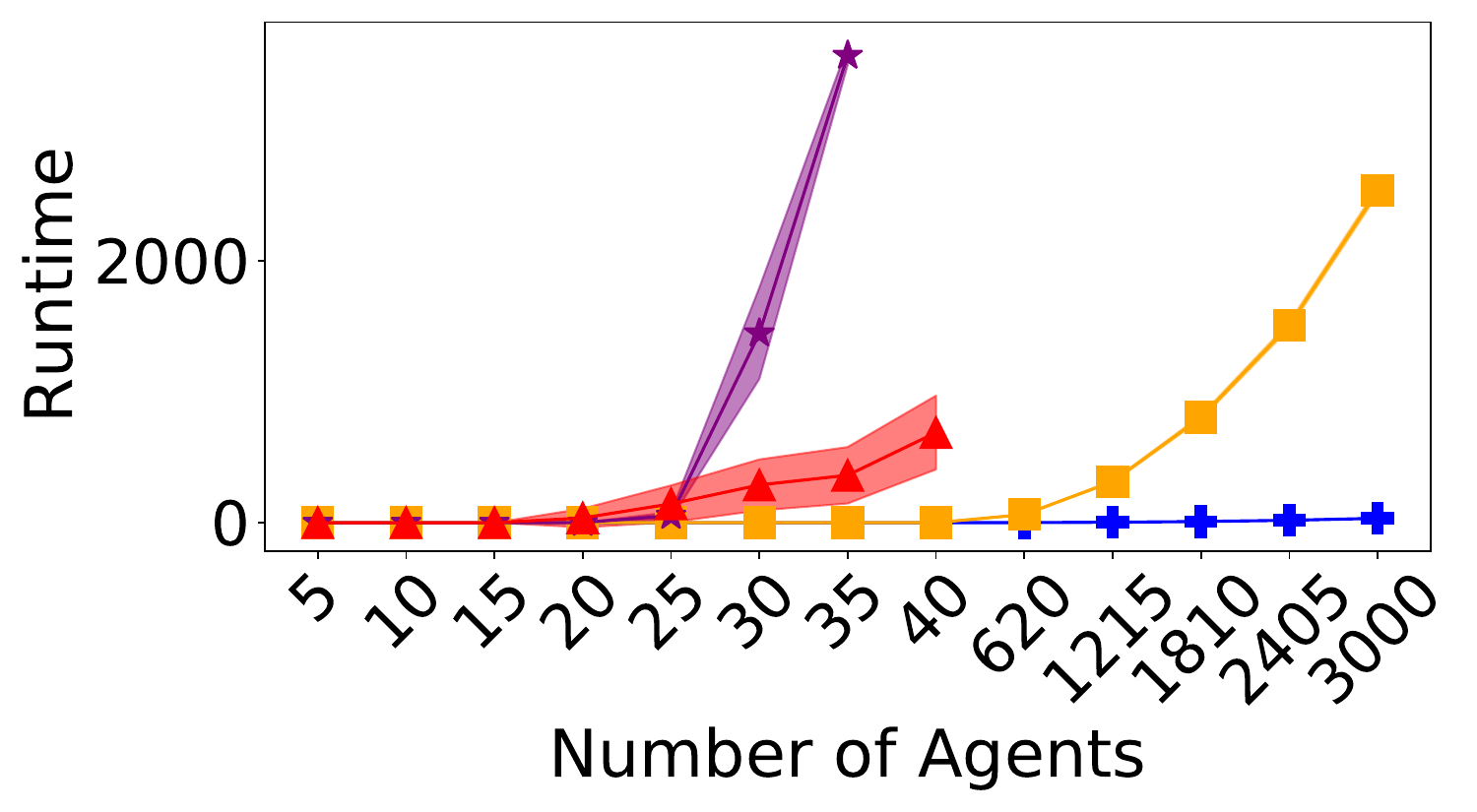}
        \label{fig:random-32-32-20-layer=1-runtime}
    \end{subfigure}
    \hfill
    \begin{subfigure}{0.246\textwidth}
        \centering
        \offinterlineskip
        \includegraphics[width=1\textwidth, height=0.6\textwidth]{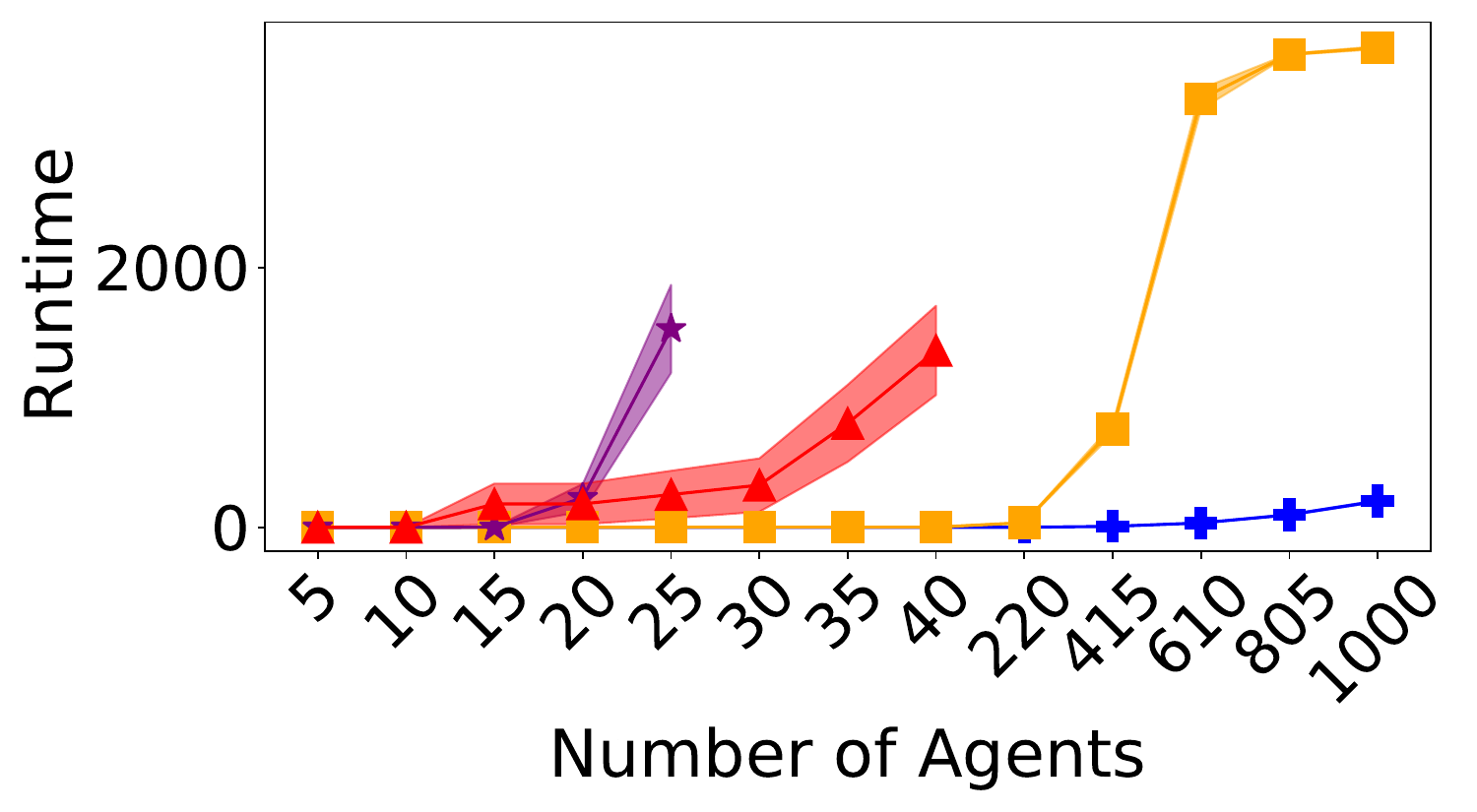}
        \label{fig:den312d-layer=1-runtime}
    \end{subfigure}
    \hfill
    \begin{subfigure}{0.246\textwidth}
        \centering
        \offinterlineskip
        \includegraphics[width=1\textwidth, height=0.6\textwidth]{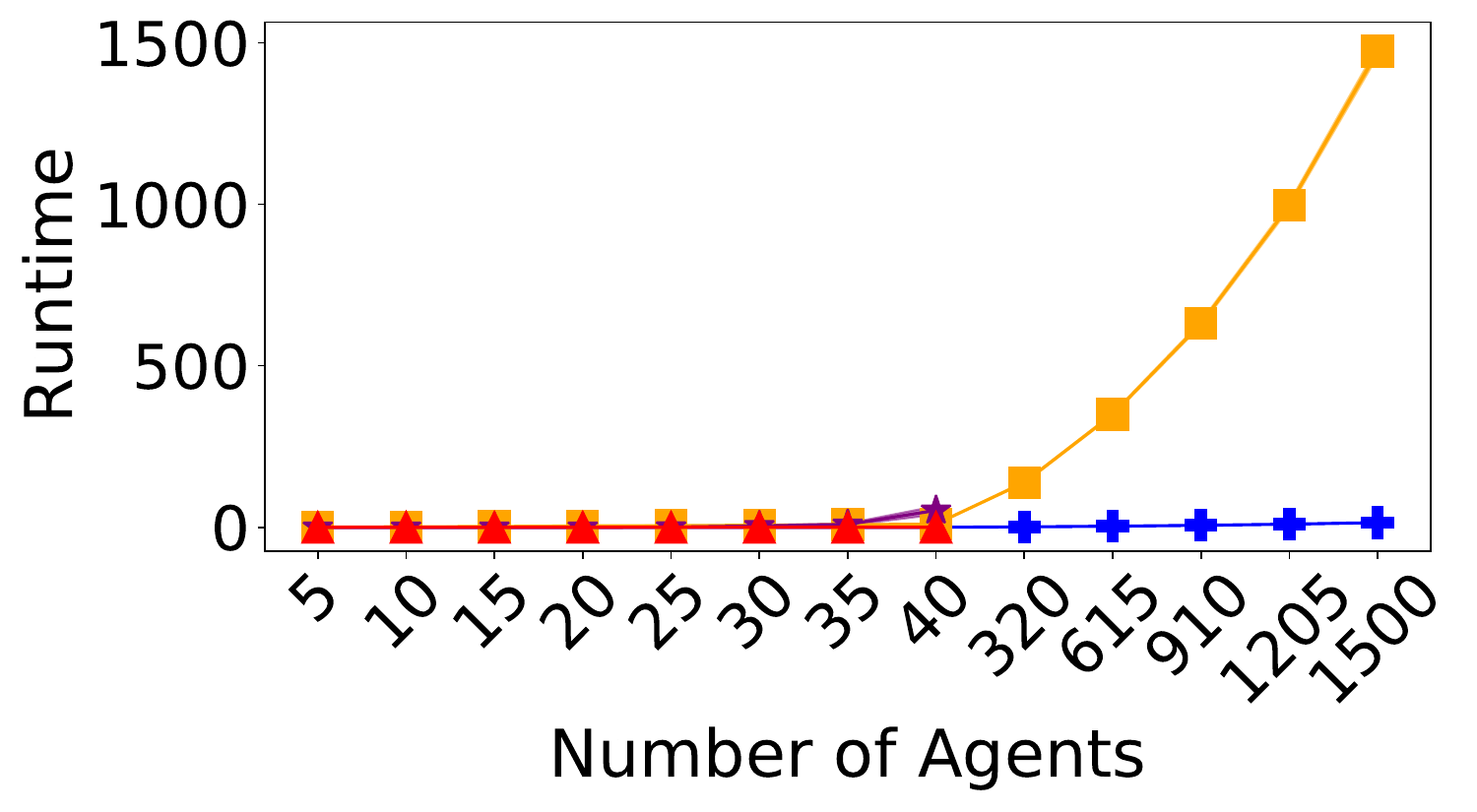}
        \label{fig:Paris-1-256-layer=1-runtime}
    \end{subfigure}
    \hfill
    \begin{subfigure}{0.246\textwidth}
        \centering
        \offinterlineskip
        \includegraphics[width=1\textwidth, height=0.6\textwidth]{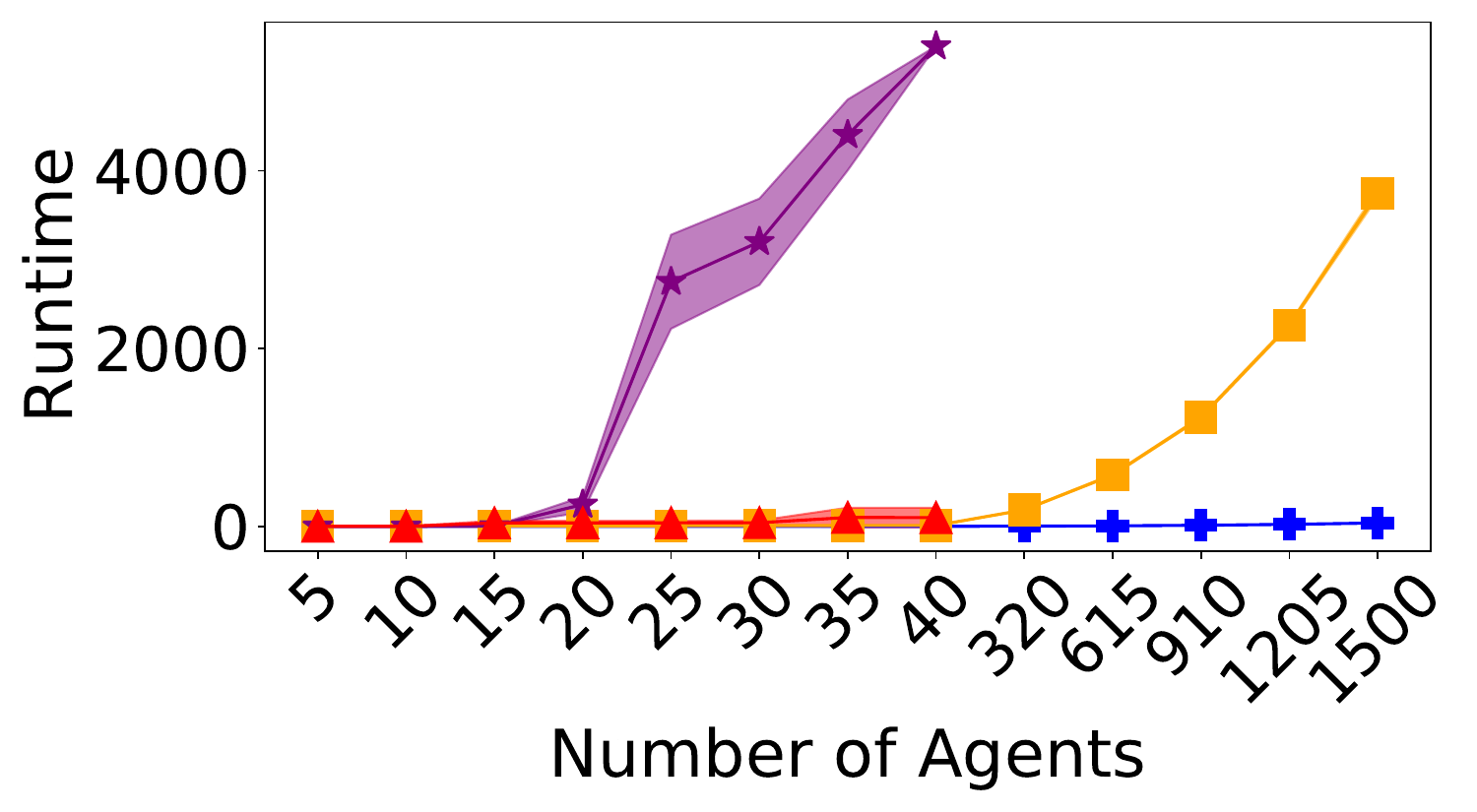}
        \label{fig:den520d-layer=1-runtime}
    \end{subfigure}
    \\
    \begin{subfigure}{0.246\textwidth}
        \centering
        \offinterlineskip
        \includegraphics[width=1\textwidth, height=0.6\textwidth]{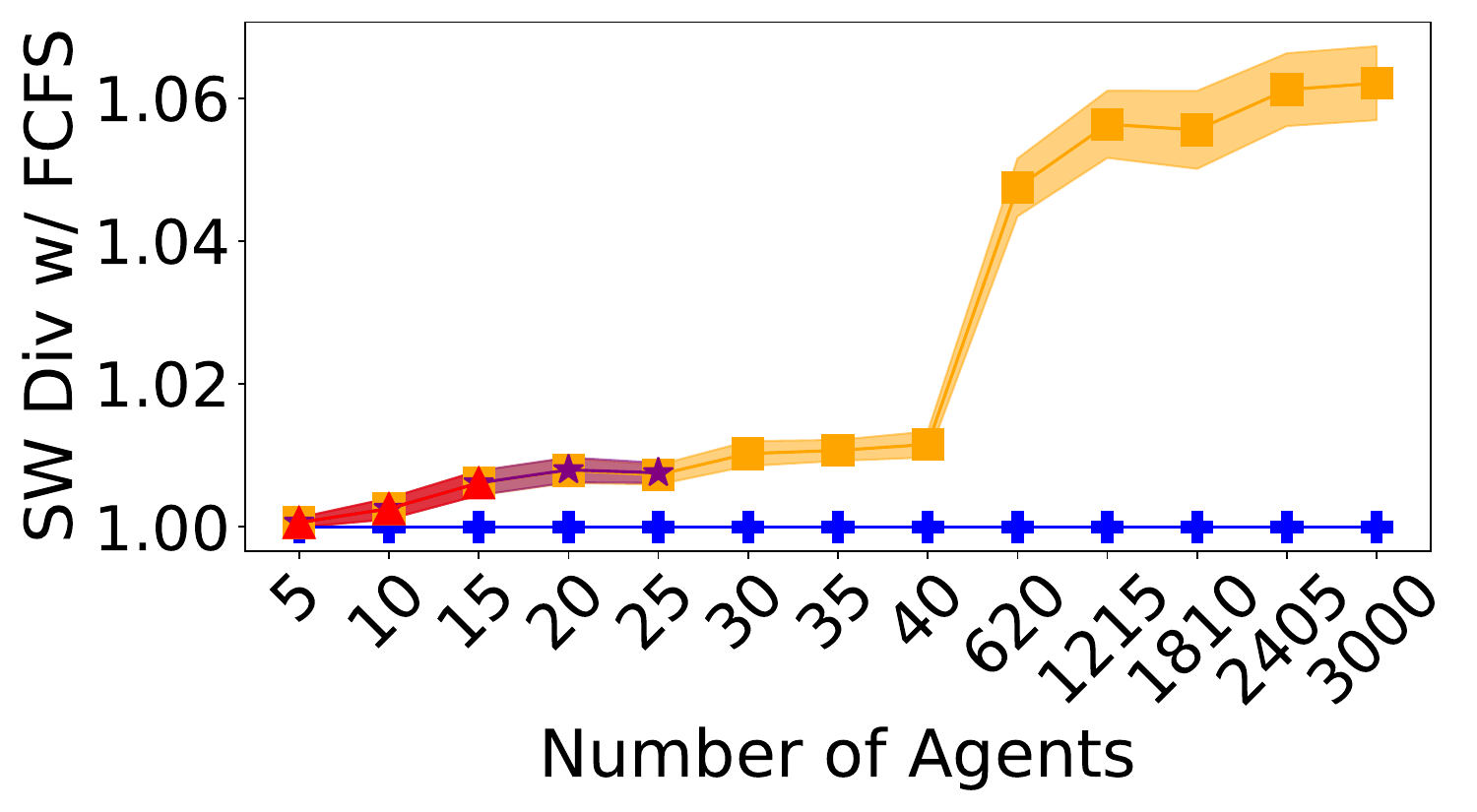}
        \label{fig:random-32-32-20-layer=1-welfare-subopt}
    \end{subfigure}
    \hfill
    \begin{subfigure}{0.246\textwidth}
        \centering
        \offinterlineskip
        \includegraphics[width=1\textwidth, height=0.6\textwidth]{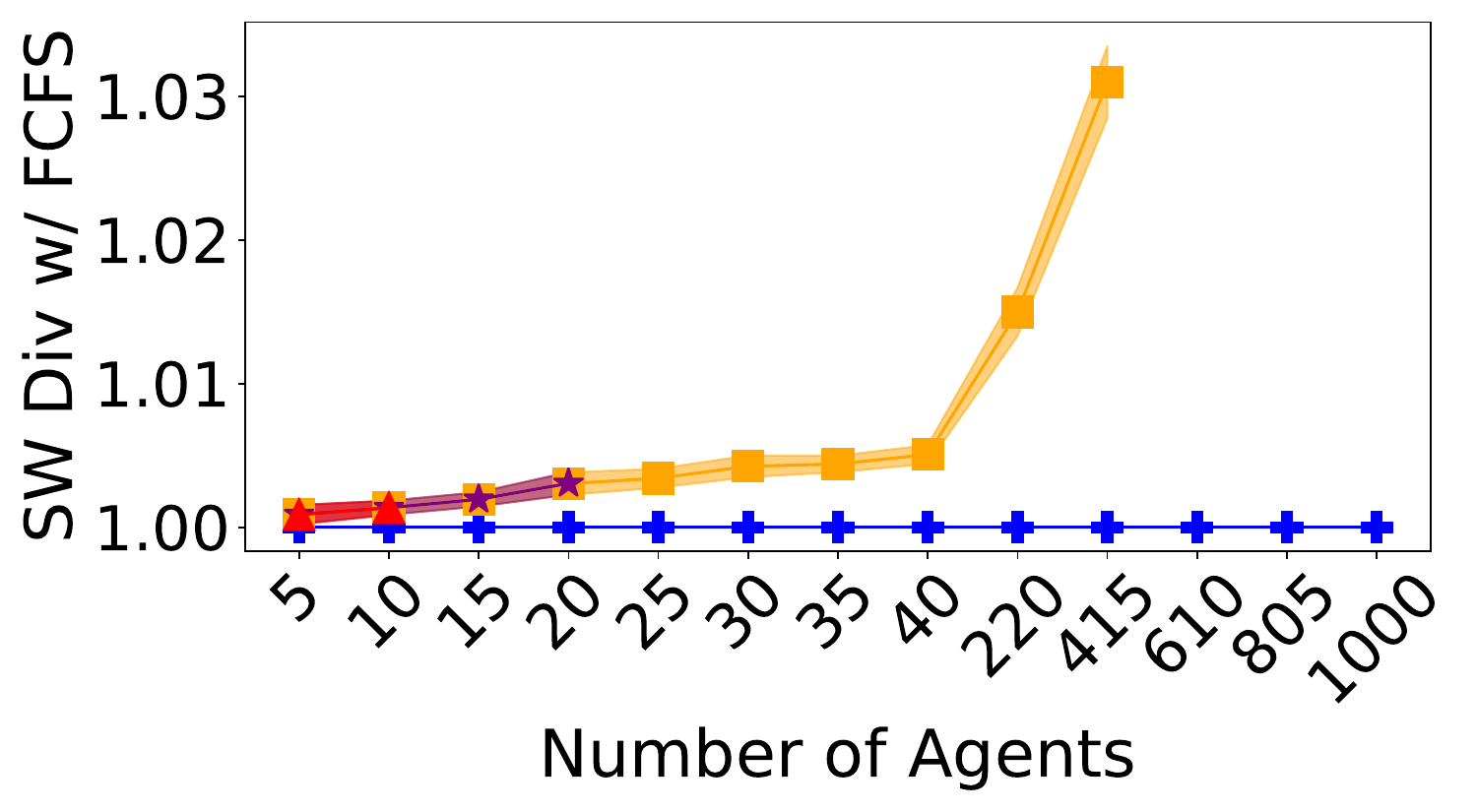}
        \label{fig:den312d-layer=1-welfare-subopt}
    \end{subfigure}
    \hfill
    \begin{subfigure}{0.246\textwidth}
        \centering
        \offinterlineskip
        \includegraphics[width=1\textwidth, height=0.6\textwidth]{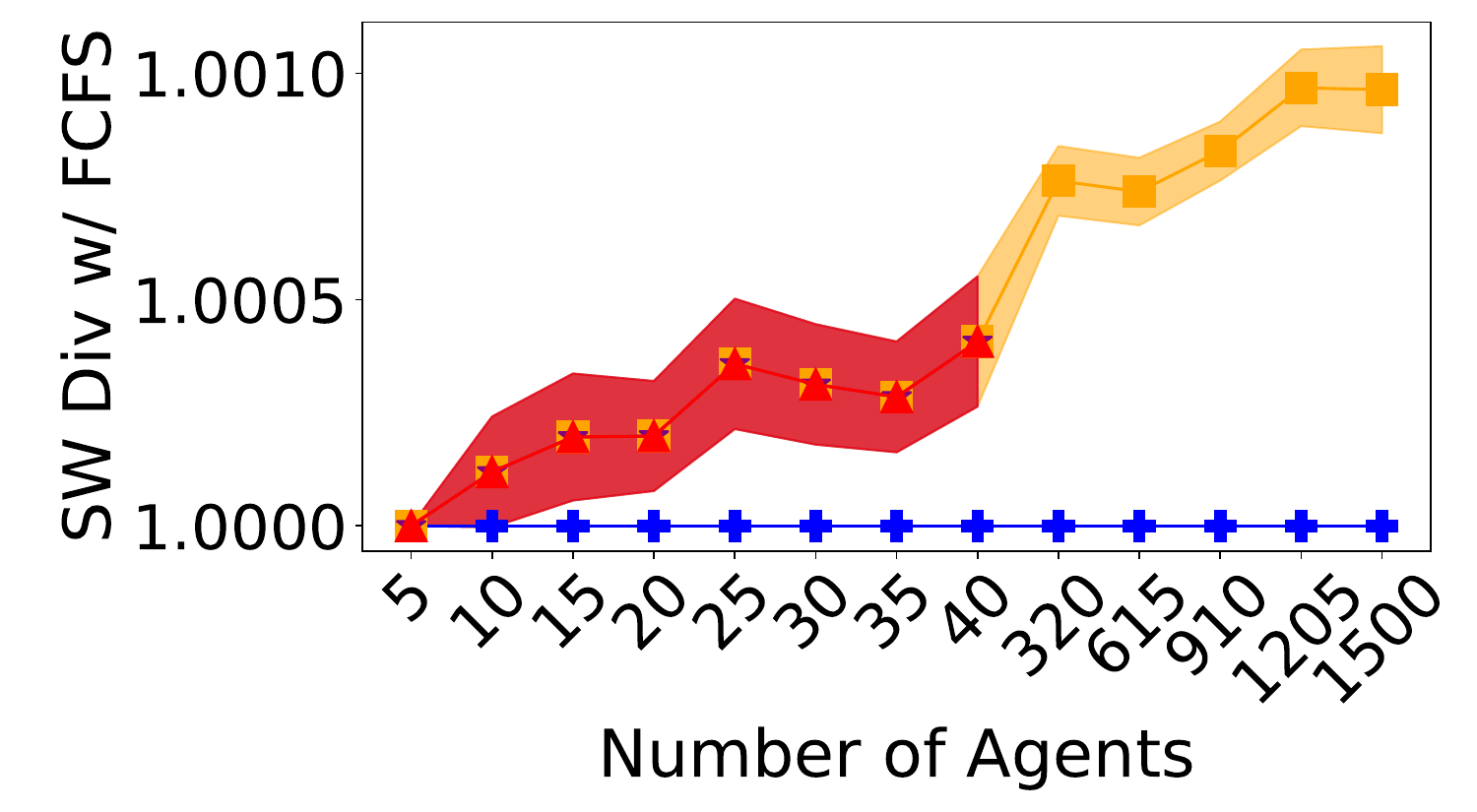}
        \label{fig:Paris-1-256-layer=1-welfare-subopt}
    \end{subfigure}
    \hfill
    \begin{subfigure}{0.246\textwidth}
        \centering
        \offinterlineskip
        \includegraphics[width=1\textwidth, height=0.6\textwidth]{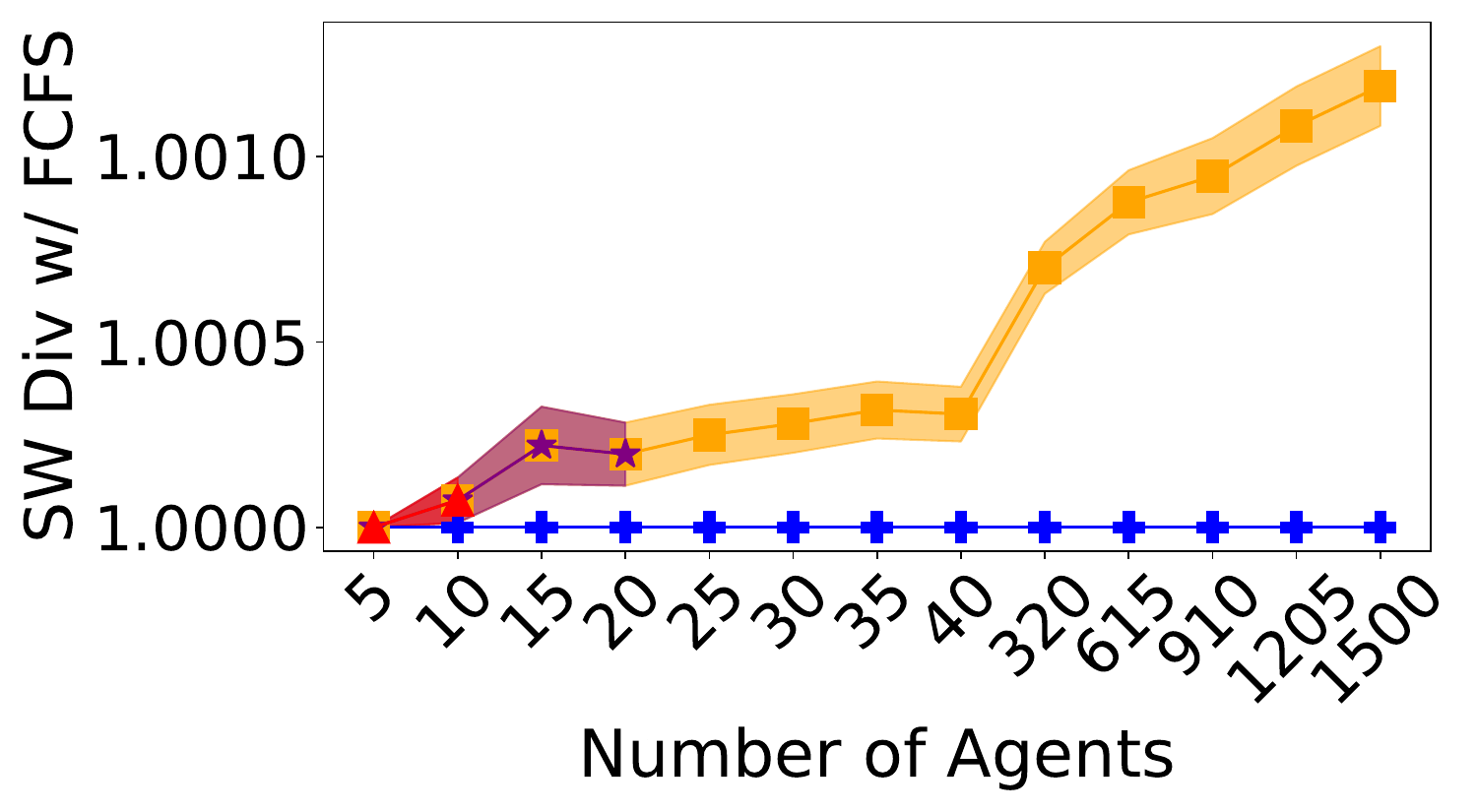}
        \label{fig:den520d-layer=1-welfare-subopt}
    \end{subfigure}
    \\
    \begin{minipage}[b]{0.24\textwidth}
        \centering
        \textbf{\textit{random-32-32-20}}
    \end{minipage}
    \hfill
    \begin{minipage}[b]{0.24\textwidth}
        \centering
        \textbf{\textit{den312d}}
    \end{minipage}
    \hfill
    \begin{minipage}[b]{0.24\textwidth}
        \centering
        \textbf{\textit{Paris\_1\_256}}
    \end{minipage}
    \hfill
    \begin{minipage}[b]{0.24\textwidth}
        \centering
        \textbf{\textit{den520d}}
    \end{minipage}\\
    \caption{Success rate, runtime, and ratio-to-baseline of social welfare (SW) for PCBS, EPBS, MCPP and FCFS (the baseline). Solid lines indicate the average value over 100 instances, while the shaded area is the 95\% confidence interval. Maximum agent numbers are 3000 for \emph{random-32-32-20}, 1000 for \emph{den312d} and 1500 for \emph{Paris\_1\_256} and \emph{den520d}. Agent numbers between 5 and 40 are shown in higher granularity than between 40 and the maximum to illustrate scaling differences.}
    \label{fig:major-result}
\end{figure*}

\Cref{fig:compare-mcpp} compares the success rate, runtime and welfare over FCFS of MCPP with different sample sizes $\numsamples$ on the \emph{random-32-32-20} map. It demonstrates MCPP's ability to trade off scalability with solution quality as a function of sample size. As the number of agents and thus congestion increases, the average runtime of finding single shortest paths does too (for all mechanisms solving MAPF!), causing MCPP's apparent superlinear scaling.

\Cref{fig:major-result} compares our three mechanisms and FCFS on the standard MAPF metrics of runtime and success rate, and shows their achieved mean welfare over all instances as the ratio when divided over FCFS's welfare. If a mechanism for a given map and number of agents fails to solve all instances, we do not plot their welfare. All mechanisms, including FCFS which does not use payments, are strategyproof. FCFS, being essentially 1-sample MCPP without payments, serves as the baseline with fastest runtime and lowest welfare. PCBS, using the optimal CBS algorithm, shows the highest welfare but fails to scale to larger instances within our time limit. EPBS, using the suboptimal PBS algorithm, scales similarly. We would expect to see more differentiation in achieved welfare of these two mechanisms at settings with more agents and thus more collisions, where allocation decisions have a bigger impact on welfare. To observe this behaviour, one could parallelize PCBS, EPBS and MCPP (see \Cref{SEC:MECHA}) or increase the runtime limit.
The figure demonstrates the exceptional scalability of MCPP compared to the search-based PCBS and EPBS while improving over the welfare baseline. The magnitude of welfare difference between the mechanisms depends on the choice of cost and value distributions and can be increased at will by e.g. employing higher variance distributions (see appendix).

\Cref{fig:paymenthist} shows the size of the VCG-based payments for a fixed mechanism, map and number of agents. We observe that payments are indeed non-negative, and overwhelmingly zero or very small. This indicates that in this setting, agents cause few externalities to others.

\begin{figure} 
    \centering
    \includegraphics[width=0.34\textwidth]{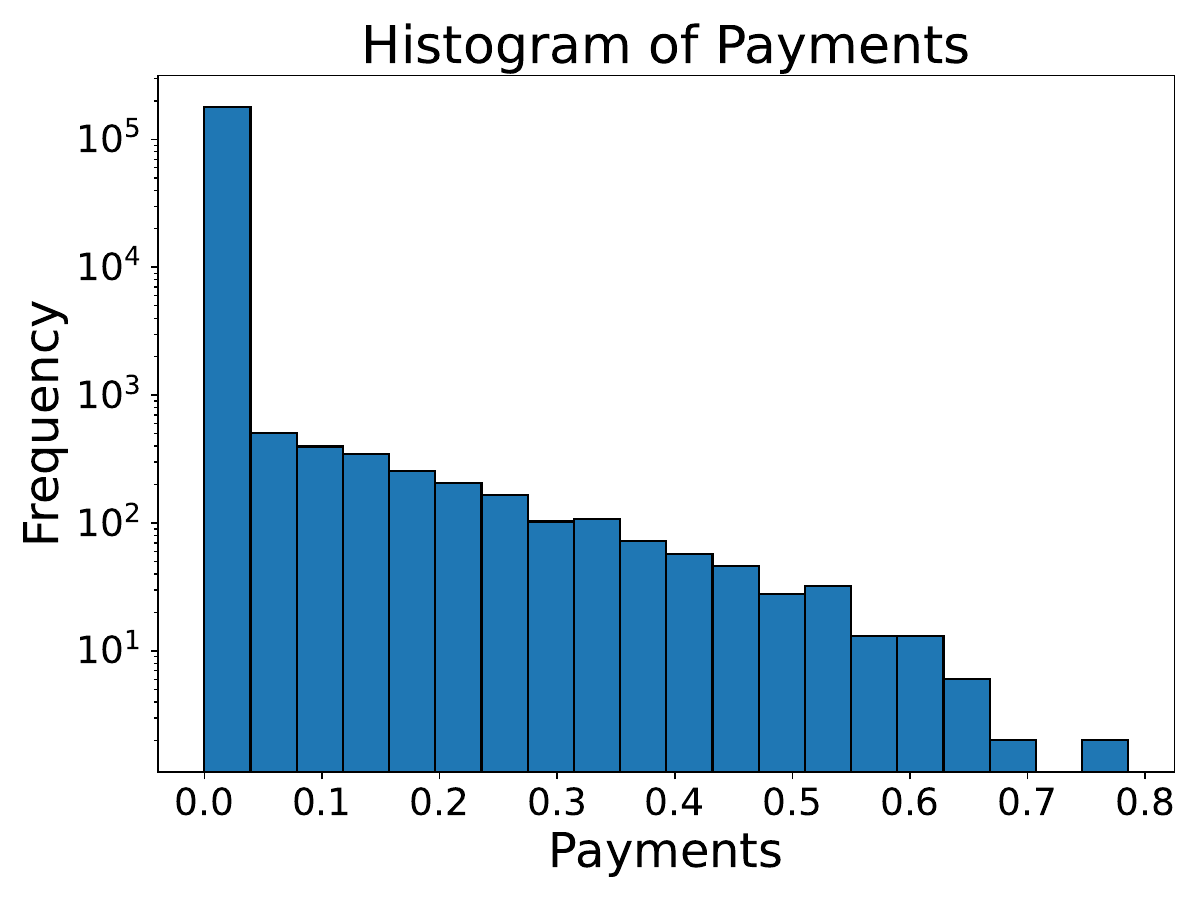}
    \caption{Size of payments for the MCPP mechanism, on the \emph{random-32-32-20} map at $\numagents=1810$ agents. Frequency in the y-axis is scaled logarithmically.}
    \label{fig:paymenthist}
\end{figure}

\section{Conclusion}\label{SEC:CONC}
We introduced a domain that bridges the gap between \emph{mechanism design} and MAPF, where paths are allocated in the standard MAPF formulation but agents might make self-interested misreports of their costs and values. 

We presented three mechanisms -- PCBS, EPBS and MCPP -- that are \emph{strategyproof}, \emph{individually rational} and have \emph{no negative payments}, and which combine MAPF allocation algorithms with a VCG-based payment rule. 

We showed how using the \emph{maximal-in-range (MIR)} property, suboptimal MAPF allocation mechanisms like EPBS and MCPP can be made strategyproof. 

Future research could find more scalable MAPF algorithms that fulfil MIR or investigate improving the welfare of suboptimal MIR-fulfilling MAPF algorithms by adjusting the range using non-misreportable agent type information.

\bibliographystyle{named}
\bibliography{ms}

\clearpage

\appendix

\section{Proofs}
\subsection{Proof of \Cref{thm:mir}}
This section follows \citet{nisan2007computationally}. 

Let $\outcomeset$ be a finite set of all allowed path allocations. For example, $\outcomeset$ may be the set of all feasible combinations of $\numagents$ paths on $(\vertexset, \edgeset)$ or a strict subset of it. Let $\outcome = (\agentpath_1,\ldots,\agentpath_\numagents) \in \outcomeset$ be an allocation, $\reportedtype_i=(\start_i,\goal_i,\reportedcost_i, \reportedval_i)$ be agent $i$'s reported type, and $\hat{\welfare}_i(\outcome):=\max(0,\reportedval_i - \reportedcost_i|\agentpath_i|)$ agent $i$'s resulting reported welfare function. We write $\reportedtype_{-i} = (\reportedtype_1,\ldots,\reportedtype_{i-1},\reportedtype_{i+1},\ldots,\reportedtype_\numagents)$ for the tuple containing all reported types excluding $\reportedtype_i$, and $\reportedtype = (\reportedtype_1,\ldots,\reportedtype_\numagents)$ or $\reportedtype = (\reportedtype_i,\reportedtype_{-i})$ for the tuple containing all reported types.

\begin{definition}
    A \emph{VCG mechanism} is a mechanism $(\allocrule,\payment)$, such that
    \begin{itemize}
        \item The allocation rule $\allocrule$ maximizes the total reported welfare over the set of allowed allocations $\outcomeset$, that is, 
        $$\forall \reportedtype=(\reportedtype_1,\ldots,\reportedtype_\numagents): f(\reportedtype) \in \argmax_{\outcome \in D} \sum_{i=1}^\numagents \reportedwelfare_i(\outcome)$$
        \item The payment rule $\payment(\reportedtype) = \bigl(\payment_1(\reportedtype),\ldots,\payment_\numagents(\reportedtype)\bigr)$ charges agent $i$ the amount according to the \emph{VCG-formula}
        $$\payment_i(\reportedtype):= -\sum_{j\neq i}\welfare_j \bigl(f(\reportedtype_i,\reportedtype_{-i})\bigr) + \vcgfunc_i(\reportedtype_{-i}),$$
    \end{itemize}
    where $\vcgfunc_i(\cdot)$ is an arbitrary function of $\reportedtype_{-i}$.
\end{definition}

If we drop the requirement that the allocation rule $f$ should maximize welfare over all allowed outcomes $\outcomeset$, we obtain \emph{VCG-based} mechanisms.
\begin{definition}
    A \emph{VCG-based mechanism} is a mechanism $(\allocrule,\payment)$, such that
    \begin{itemize}
        \item The allocation rule $\allocrule$ maps reported types into the set of allowed outcomes $\outcomeset$.
        \item The payment rule charges agent $i$ the amount according to the \emph{VCG-formula}.
    \end{itemize}
\end{definition}

The utility of an agent $i$ under a VCG-based mechanism is given by $\utility_i(\reportedtype) = \welfare_i\bigl(f(\reportedtype)\bigr) - \payment_i(\reportedtype)$. By inserting the payment definition via the VCG-formula, the utility becomes
$$\utility_i(\reportedtype_i,\reportedtype_{-i}) = \welfare_i\bigl(\allocrule(\reportedtype_i,\reportedtype_{-i})\bigr) + \sum_{j \neq i} \reportedwelfare_j \bigl(f(\reportedtype_i,\reportedtype_{-i})\bigr)  - \vcgfunc(\reportedtype_{-i}).$$

\begin{proposition}\label{thm:vcg-sp}
    Let $(\allocrule,\payment)$ be a VCG-based mechanism. Assuming that $\allocrule$ produces the outcome which maximizes reported welfare over $\outcomeset$ (i.e., $(\allocrule,\payment)$ is a VCG-mechanism), then each agent can maximize her utility by truthfully reporting $\reportedtype_i := \agenttype_i$, such that $\reportedwelfare_i = \welfare_i$.
\end{proposition}
\begin{proof}
    To see this, assume that agent $i$ reports truthfully, meaning that $\reportedtype_i := \agenttype_i$ and thus $\reportedwelfare_i = \welfare_i$. 
    For $\outcome\in\outcomeset$, define $g(\outcome):= \welfare_i(\outcome) + \sum_{j \neq i} \reportedwelfare_j (\outcome).$ Fixing all other agents' reports $\reportedtype_{-i}$, for all possible misreports $\agenttype_i' \neq \agenttype_i$ which agent $i$ could make,
    \begin{equation*}
    \begin{split}
        \utility_i(\agenttype_i,\reportedtype_{-i}) &- \utility_i(\agenttype_i',\reportedtype_{-i}) \\
        &= \welfare_i(\allocrule(\agenttype_i,\reportedtype_{-i})) + \sum_{j \neq i} \reportedwelfare_j \bigl(f(\agenttype_i,\reportedtype_{-i})\bigr) \\
        &- \welfare_i(\allocrule(\agenttype_i',\reportedtype_{-i})) + \sum_{j \neq i} \reportedwelfare_j \bigl(f(\agenttype_i',\reportedtype_{-i})\bigr)\\
        &- \bigl(\vcgfunc_i(\reportedtype_{-i}) - \vcgfunc_i(\reportedtype_{-i})\bigr)\\
        &= g\bigl(\allocrule(\agenttype_i,\reportedtype_{-i})\bigr) - g\bigl(\allocrule(\agenttype_i',\reportedtype_{-i})\bigr)\\
        &\geq 0,
    \end{split}
    \end{equation*}
    where the inequality follows by $\allocrule$'s welfare maximisation $\allocrule(\agenttype_i,\reportedtype_{-i}) \in \argmax_{\outcome \in \outcomeset} g(\outcome)$ and since $\allocrule(\agenttype_i',\reportedtype_{-i}) \in \outcomeset$. Therefore, reporting truthfully always maximizes an individual agent's utility and thus $(\allocrule,\payment)$ is strategyproof.
\end{proof}
We now prove \Cref{thm:mir}, recalling that we defined the VCG-based payment rule as
    $$\payment_i(\reportedtype):= \reportedwelfare_{-i}(\outcome^*_{-i})-\reportedwelfare_{-i}(\outcome^*),$$
    where
    \begin{align*}
        \outcome^\ast &:= \argmax_{\outcome \in \range} \sum_i \reportedwelfare_i (\outcome) \\
        \outcome_{-i}^\ast &:= \argmax_{\outcome\in\range} \sum_{j \neq i} \reportedwelfare_j (\outcome).
    \end{align*}
\mir*
\begin{proof}
    An agent's utility is then
    \begin{align*} 
        \utility_i(\reportedtype) &= \welfare_i\bigl(\allocrule(\reportedtype)\bigr) - \payment_i\bigl(\allocrule(\reportedtype)\bigr)\\
        &= \welfare_i\bigl(\allocrule(\reportedtype)\bigr) + \sum_{j \neq i} \reportedwelfare_j(\outcome^\ast) - \sum_{j \neq i} \reportedwelfare_j(\outcome_{-i}^\ast).
    \end{align*}
    The mechanism $(\allocrule,\payment)$ thus fulfills the definition of a \emph{VCG-based} mechanism where 
    $$\vcgfunc_i(\reportedtype_{-i}) := \sum_{j \neq i} \reportedwelfare_j (\outcome_{-i}^\ast).$$
    By definition of MIR, $\allocrule$ produces the outcome which maximizes reported welfare over $\range$, which in particular ensures that all outputs of $\allocrule$ lie within $\range$. The mechanism $(\allocrule,\payment)$ is thus a \emph{VCG mechanism}, whose set of allowable outputs $\outcomeset$ is the range $\range$. Following \Cref{thm:vcg-sp}'s calculation with $\allocrule(\agenttype_i,\reportedtype_{-i}) \in \argmax_{\outcome \in \range} g(\outcome)$ and $\allocrule(\agenttype_i',\reportedtype_{-i}) \in \range$, we see that agents which report truthfully always maximize their utility.
\end{proof}

\subsection{Proof of \Cref{thm:ir-negpayments}}
\irnegpayments*
\begin{proof}
    Proving (i): We show that $(\allocrule,\payment)$ ensures that each agent $i$ has a non-negative utility $\utility_i = \welfare_i - \payment_i \geq 0$ for all possible agent reports $\reportedtype$:

    \begin{align}
        \utility_i(\reportedtype) &= \reportedwelfare_i(\outcome^\ast) - \Bigl(\sum_{j \neq i} \reportedwelfare_j(\outcome^\ast_{-i}) - \sum_{j \neq i} \reportedwelfare_j(\outcome^\ast)\Bigr) \\
        &= \sum_{i} \reportedwelfare_i(\outcome^\ast) - \sum_{j \neq i} \reportedwelfare_j(\outcome^\ast_{-i}) \\
        &\geq \sum_{i} \reportedwelfare_i(\outcome^\ast) - \sum_{j \neq i} \reportedwelfare_j(\outcome^\ast_{-i}) - \reportedwelfare_i(\outcome^\ast_{-i}) \label{eq:epbs-ir-1} \\ 
        &= \sum_{i} \reportedwelfare_i(\outcome^\ast) - \sum_{i} \reportedwelfare_i(\outcome^\ast_{-i}) \geq 0. \label{eq:epbs-ir-2}
    \end{align}

    Where \cref{eq:epbs-ir-1} uses that due to our definition of $\cost_i(\agentpath_i)$, an agent's (reported) welfare is always non-negative, and non-negativity in \cref{eq:epbs-ir-2} follows due to $\outcome^\ast$ by definition being the assignment which maximizes $\sum_i \reportedwelfare_i(\outcome)$ for all $\outcome\in\outcomeset$.

    Proving (ii): We have for all $i$ and reported types $\reportedtype$:
    \begin{align*} 
        \payment_i (\reportedtype) &= \reportedwelfare_{-i} (\outcome^\ast_{-i}) - \reportedwelfare_{-i}(\outcome^\ast) \\
        &= \sum_{j \neq i} \reportedwelfare_j(\outcome^\ast_{-i}) - \sum_{j \neq i} \reportedwelfare_j(\outcome^\ast) \geq 0
    \end{align*}
    since $\outcome^\ast_{-i}$ is by definition the assignment that maximizes $\sum_{j \neq i} \reportedwelfare_j(\outcome)$ for all $\outcome \in \outcomeset$.
\end{proof}

\section{Additional Results}
In \Cref{fig:major-result-3D}, we test our mechanisms in 3D maps, modeling a UAV airspace assignment problem. We build the 3D maps by stacking the corresponding 2D maps from the MAPF benchmark~\citet{sternMultiAgentPathfindingDefinitions2019} that we use in our other experiments vertically, resulting in 3D maps of height five. MCPP with 100 samples proves its scalability by providing solutions with a 100\% success rate up to thousands of agents on the \emph{3D random-32-32-20} and \emph{3D Paris-1-256} maps with a runtime limit of 3600s. It considerably improves welfare over FCFS while being close to optimal for lower agent levels. Due to the added degrees of freedom in 3D MAPF problems, higher amounts of agents are required to generate sufficiently many collisions for FCFS' random priority ordering allocation to be notably suboptimal. This is illustrated by the very small welfare differential between FCFS and the (optimal) PCBS in the \emph{3D Paris-1-256} map even with many agents, and the larger welfare gain in the much more constrained \emph{3D random-16-16} map, obtained by truncating the \emph{3D random-32-32-20} map. The experiment for \emph{3D random-16-16}, unlike all others in this paper, was run on an AMD Ryzen Threadripper 3990X with 192 GB of RAM. 

\Cref{fig:cost-dist_vs_welfare} illustrates how choosing non-uniform cost and value distributions and scaling the distributions differently allows one to increase the magnitude of the welfare differential between our mechanisms (illustrated by MCPP) over FCFS at will. Uniform distributions, which we use for the main contents' experiments, are a standard assumption in MAPF and Mechanism Design. In this experiment on \emph{random-32-32-20} at 1800 agents, we fix the value distribution as log-normal $\mathcal{LN}(0,1)$ and vary the cost distribution from uniform to $\mathcal{LN}(0,1)$, $\mathcal{LN}(0,0.75^2)$, $\mathcal{LN}(0,0.5^2)$ and $\mathcal{LN}(0,0.25^2)$.

\begin{figure*}[!t]
    \centering
    \offinterlineskip
    \includegraphics[width=0.85\textwidth]{fig/legend.pdf}

    \label{fig:legend-3d}
    \begin{subfigure}{0.42\textwidth}
        \centering
        \offinterlineskip
        \includegraphics[width=1\textwidth]{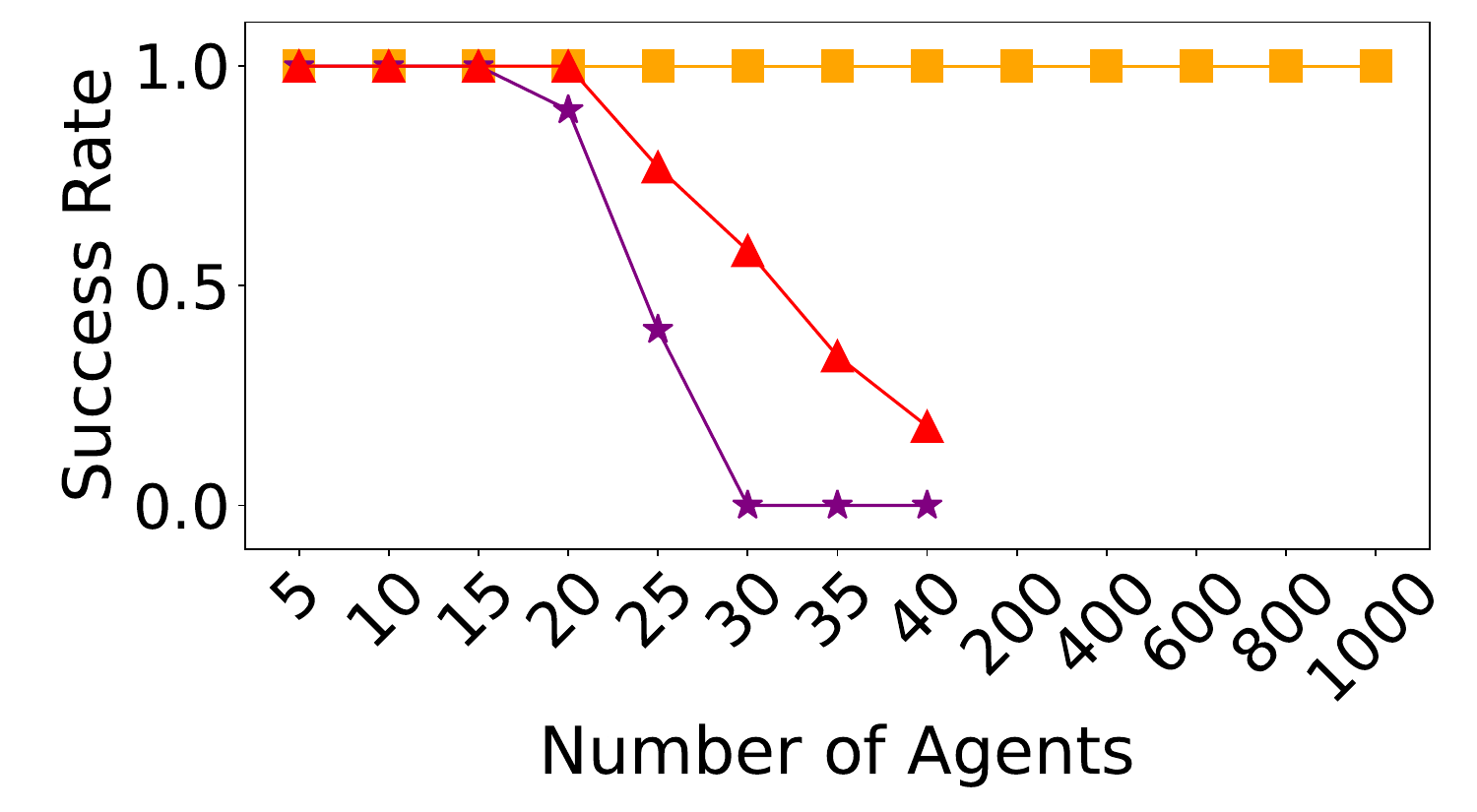}
        \label{fig:random-16-16-layer=5-success}
    \end{subfigure}
    \hspace{2em}
    \begin{subfigure}{0.42\textwidth}
        \centering
        \offinterlineskip
        \includegraphics[width=1\textwidth]{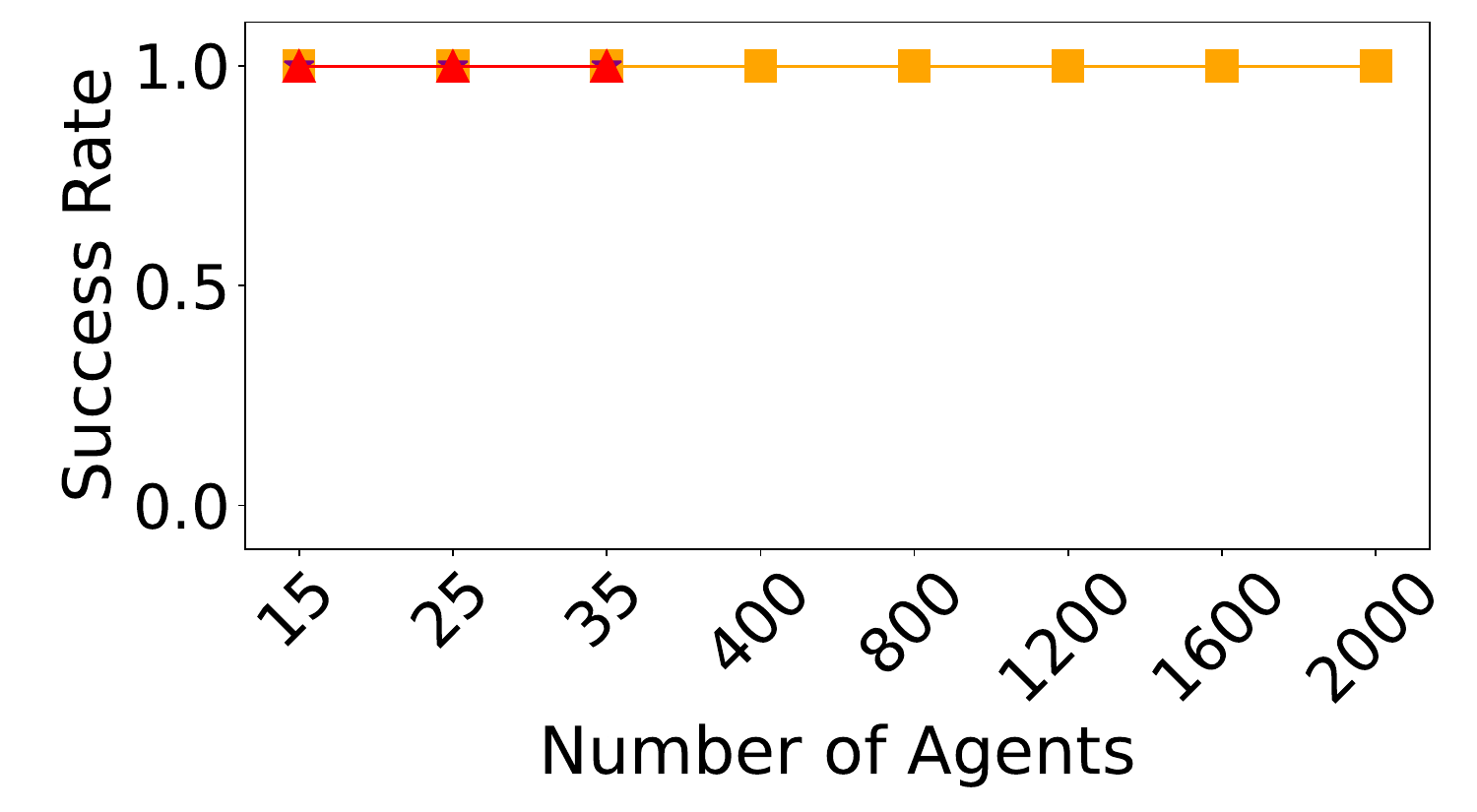}
        \label{fig:Paris-1-256-layer=5-success}
    \end{subfigure}
    \\
    \begin{subfigure}{0.42\textwidth}
        \centering
        \offinterlineskip
        \includegraphics[width=1\textwidth]{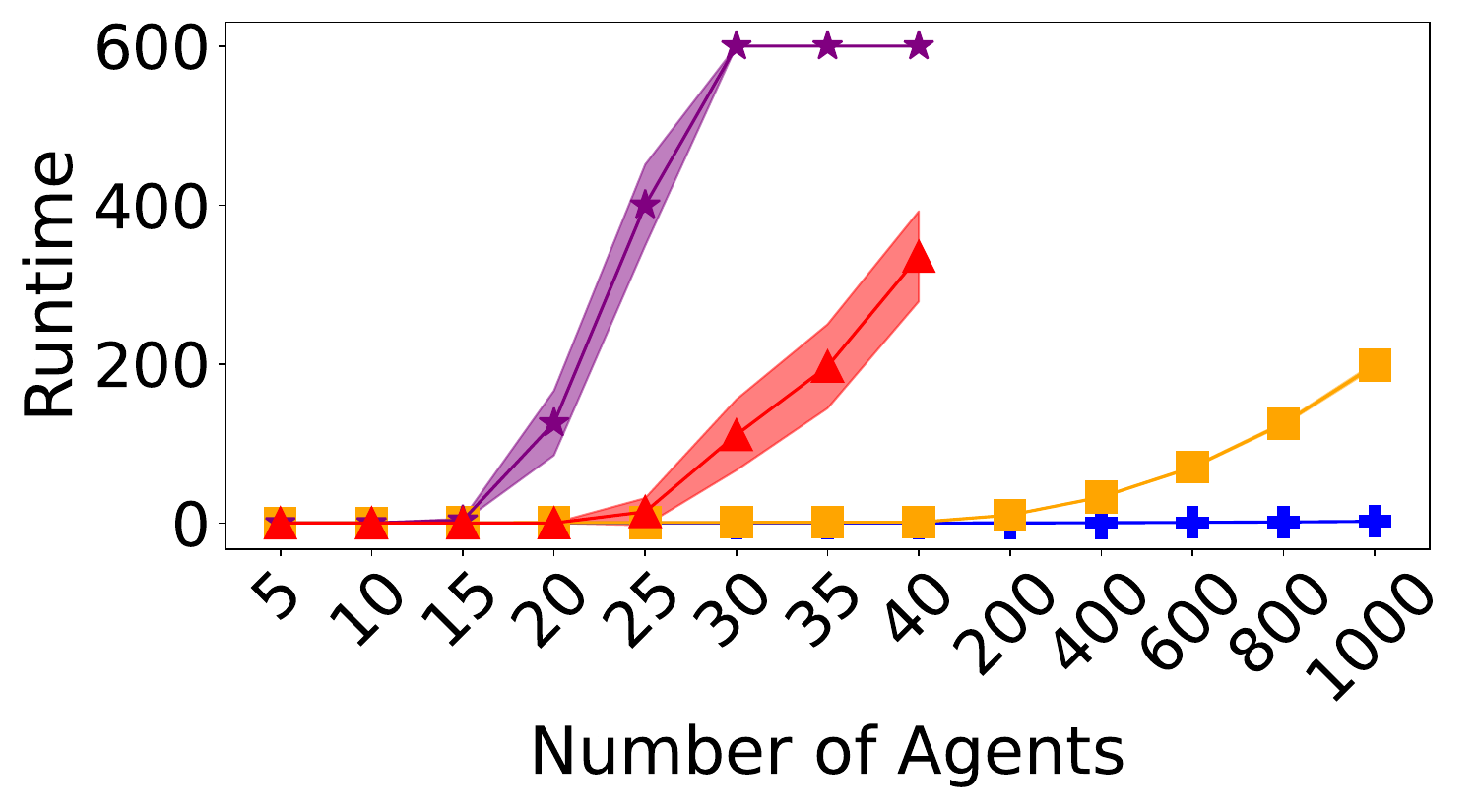}
        \label{fig:random-16-16-layer=5-runtime}
    \end{subfigure}
    \hspace{2em}
    \begin{subfigure}{0.42\textwidth}
        \centering
        \offinterlineskip
        \includegraphics[width=1\textwidth]{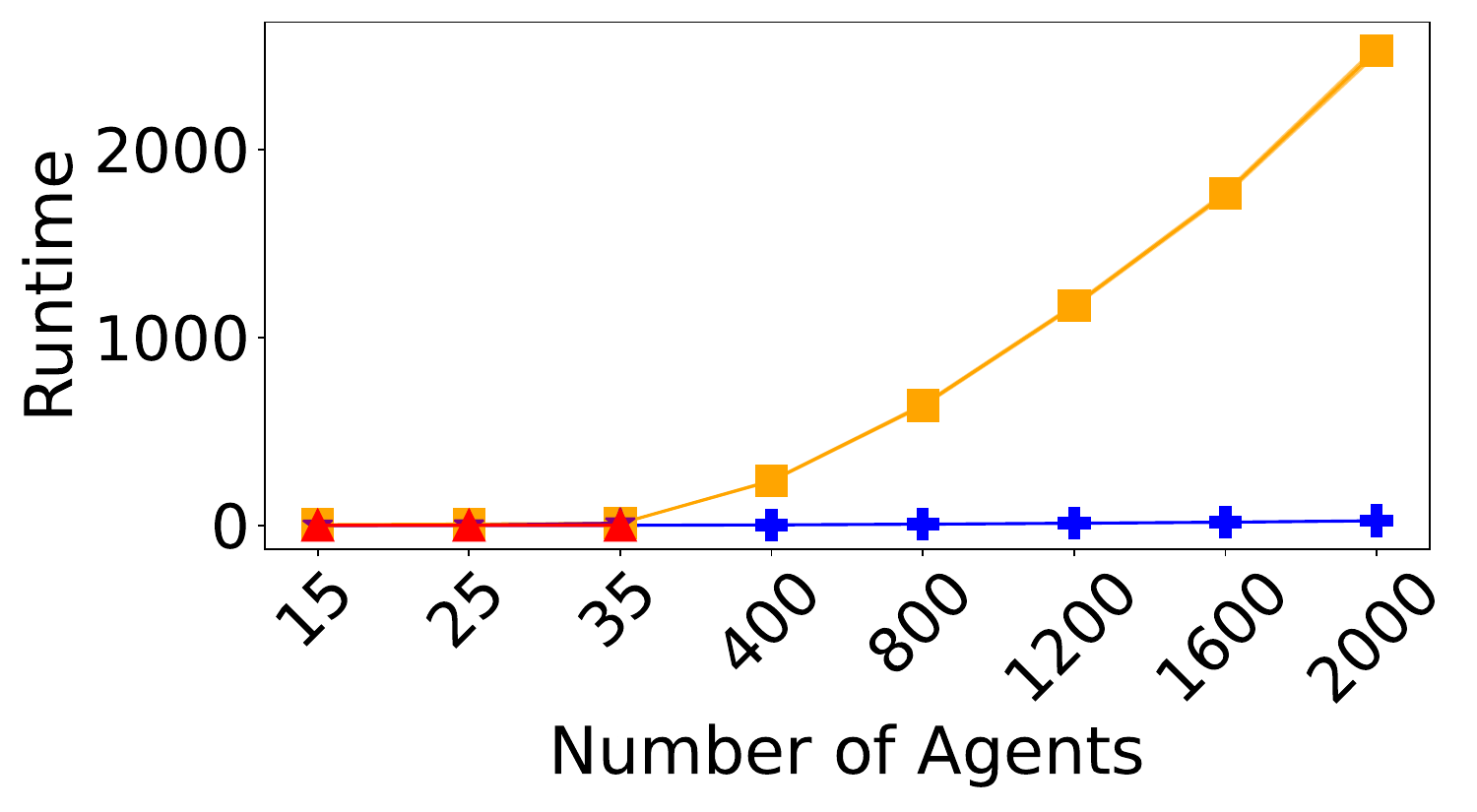}
        \label{fig:Paris-1-256-layer=5-runtime}
    \end{subfigure}
    \\
    \begin{subfigure}{0.42\textwidth}
        \centering
        \offinterlineskip
        \includegraphics[width=1\textwidth]{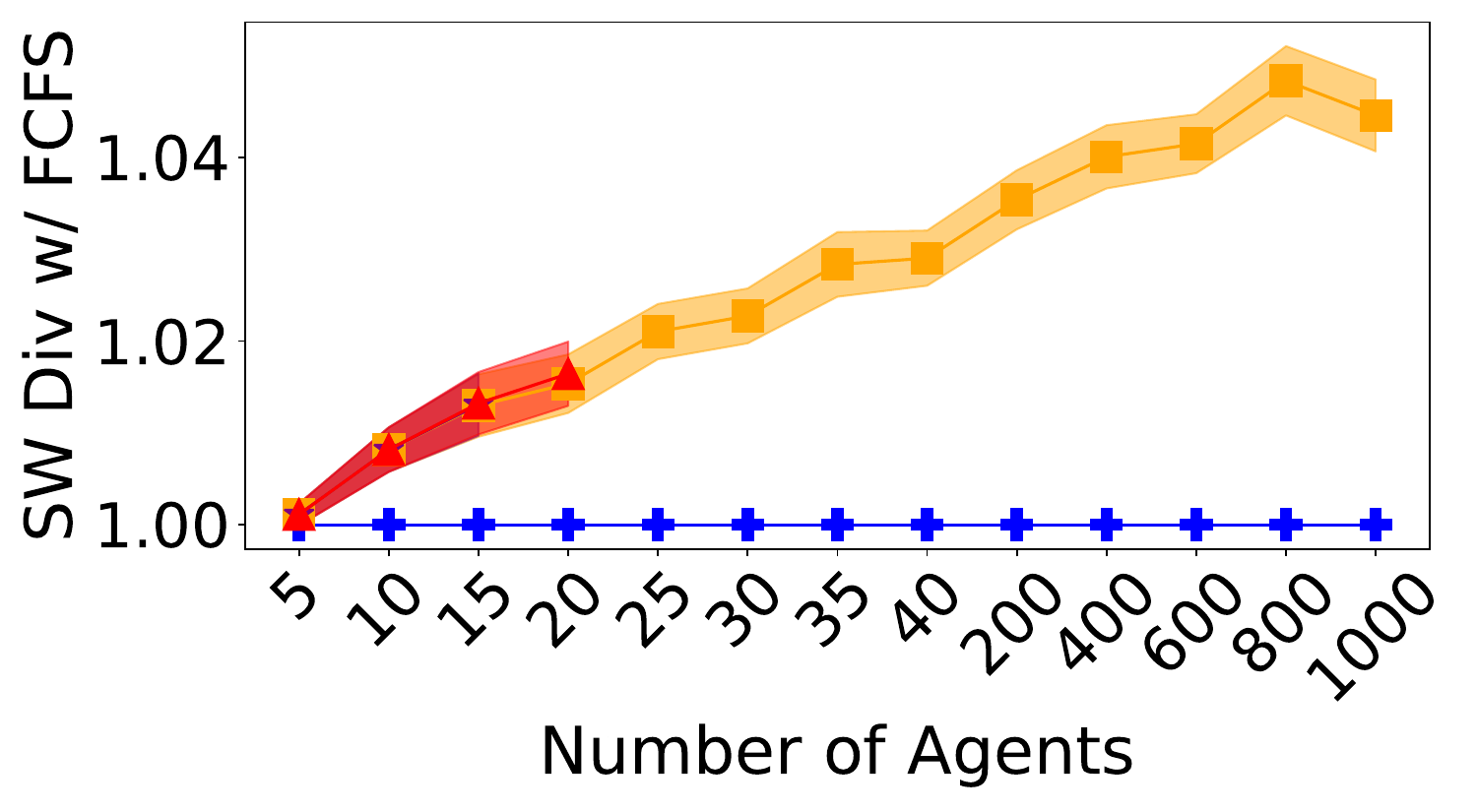}
        \label{fig:random-16-16-layer=5-welfare-subopt}
    \end{subfigure}
    \hspace{2em}
    \begin{subfigure}{0.4\textwidth}
        \centering
        \offinterlineskip
        \includegraphics[width=1\textwidth]{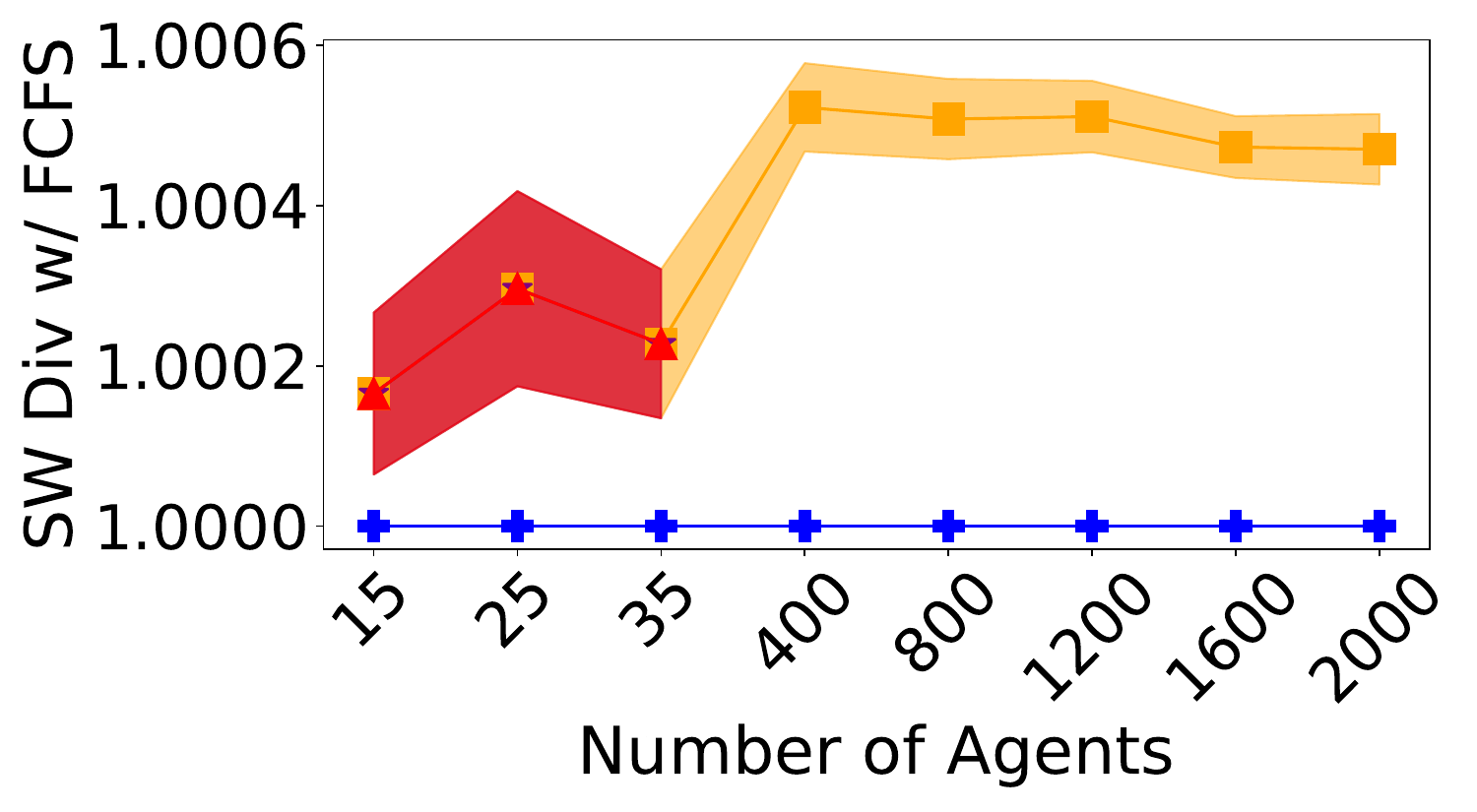}
        \label{fig:Paris-1-256-layer=5-welfare-subopt}
    \end{subfigure}
    \\
    \begin{minipage}[b]{0.4\textwidth}
        \centering
        \textbf{3D \textit{random-16-16}}
    \end{minipage}
    \hspace{2em}
    \begin{minipage}[b]{0.4\textwidth}
        \centering
        \textbf{3D \textit{Paris\_1\_256}}
    \end{minipage}\\
    \caption{Success rate, runtime, and ratio-to-baseline of social welfare for PCBS, EPBS, MCPP and FCFS (the baseline) in 3D versions of our benchmark maps. Solid lines indicate the average value over 100 instances, while the shaded area is the 95\% confidence interval. If a mechanism fails to achieve a 100\% success rate, we do not plot their welfare for that number of agents.}
    \label{fig:major-result-3D}
\end{figure*}

\begin{figure}
    \centering
    \includegraphics[width=0.45\textwidth]{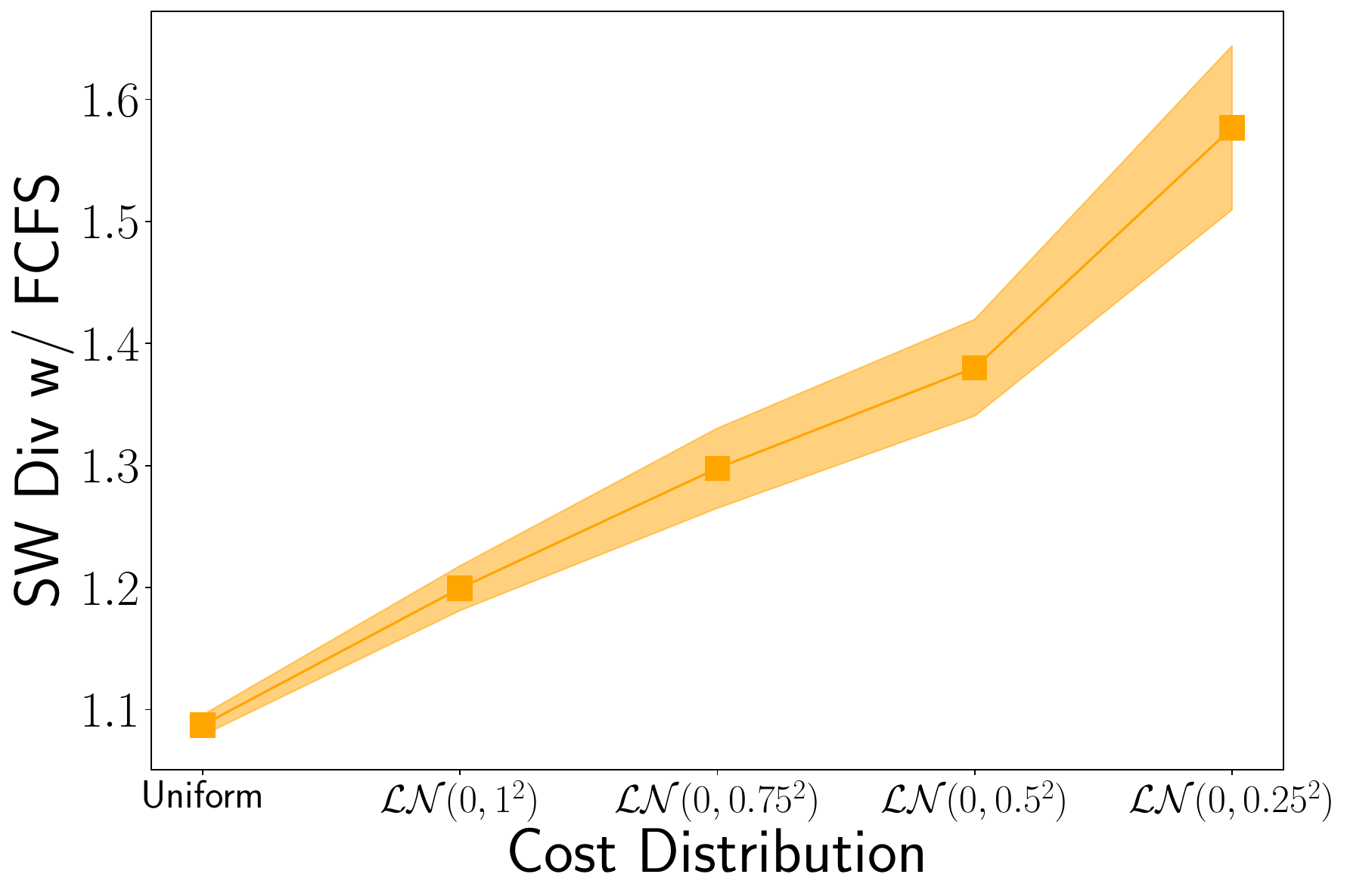}
    \caption{Social welfare suboptimality in \textit{random-32-32-20} map with MCPP, at 1800 agents. All runs use a standard log-normal value distribution. We vary their cost distributions from uniform ones to differently scaled log-normal ones.}
    \label{fig:cost-dist_vs_welfare}
\end{figure}
\end{document}